\documentclass{article}

\usepackage[final,nonatbib]{neurips_2019}

\usepackage[utf8]{inputenc} 
\usepackage[T1]{fontenc}    
\usepackage{hyperref}       
\usepackage{url}            
\usepackage{booktabs}       
\usepackage{amsfonts}       
\usepackage{nicefrac}       
\usepackage{microtype}      

\usepackage{amsmath,amssymb,amsthm}

\usepackage{algorithm}
\usepackage{algorithmic}
\usepackage{amsmath}

\usepackage{graphicx}
\usepackage{caption}
\usepackage{subfigure} 

\newtheorem{theorem}{Theorem}
\newtheorem{corollary}{Corollary}
\newtheorem{proposition}{Proposition}
\newtheorem{lemma}{Lemma}
\newtheorem{assumption}{Assumption}

\newtheorem{definition}{Definition}

\def\R{{\mathbb R}}
\def\CE{{\mathbb E}}

\def\tp{{\tilde{p}}}
\def\tx{{\tilde{x}}}
\def\tDelta{{\tilde{\Delta}}}
\def\mG{\mathcal{G}}
\def\te{{\tilde{e}}}
\def\mC{\mathcal{C}}

\graphicspath{{images/}} 

\title{Communication-Efficient Distributed Blockwise Momentum SGD 
with Error-Feedback}

\author{
  Shuai Zheng\thanks{The work was done before Shuai Zheng joined Amazon Web Services. } $^{ \hspace{.02in}1,2}$, Ziyue Huang$^{1}$, James T. Kwok$^{1}$\\
  \texttt{shzheng@amazon.com, \{zhuangbq, jamesk\}@cse.ust.hk} \\
  $^{1}$Department of Computer Science and Engineering\\
  Hong Kong University of Science and
  Technology\\
  $^{2}$Amazon Web Services 
}

\begin{document}

\maketitle

\begin{abstract}
Communication overhead is a major bottleneck hampering the scalability of
distributed machine learning systems.  
Recently, there has been a surge of interest in 
using gradient compression
to improve the communication efficiency of distributed 
neural network
training. 
Using 1-bit quantization,
signSGD with majority vote 
achieves a 32x reduction on communication cost. However, its convergence is based
on unrealistic assumptions and can diverge in practice. 
In this paper, we propose a general distributed compressed SGD with
Nesterov's momentum. We consider two-way
compression, which 
compresses
the gradients 
both 
to and from workers.
Convergence analysis on nonconvex problems for
general gradient compressors
is provided. 
By partitioning 
the gradient 
into blocks, 
a blockwise compressor is introduced such that 
each gradient block is compressed and transmitted in 1-bit format with a scaling
factor, leading to a nearly 32x reduction on communication. 
Experimental results show that the proposed method converges as fast as
full-precision distributed momentum SGD and achieves the same testing accuracy.  In
particular,
on distributed ResNet training with 7 workers
on the ImageNet, 
the proposed algorithm achieves the same testing accuracy as
momentum SGD using full-precision gradients, but with $46\%$ less wall clock time. 
\end{abstract}

\section{Introduction}

Deep neural networks have been highly successful in 
recent years \cite{graves2013generating,he2016identity,silver2016mastering,vaswani2017attention,zaremba2014recurrent}.
To achieve state-of-the-art performance, they often have to 
leverage the computing power of multiple 
machines during training \cite{dean-12,xing2015petuum,zinkevich2010parallelized,chen2016revisiting}.  
Popular approaches include distributed synchronous SGD and its momentum variant
SGDM, 
in
which the computational load for evaluating a mini-batch gradient is distributed
among the workers. 
Each worker performs local computation, 
and these local informations are then merged by the server for final update on the model parameters. 
However, its scalability 
is limited by the
possibly
overwhelming 
cost due to
communication of the gradient and model parameter \cite{li2014communication}. 
Let $d$ be the 
gradient/parameter
dimensionality, and $M$ be the number of workers. 
$64Md$ bits need to be transferred between the workers and server
in each iteration.

To mitigate this communication bottleneck,
the two common approaches are gradient sparsification 
and gradient quantization. 
Gradient sparsification only sends the most significant,
information-preserving
gradient entries.
A heuristic algorithm is first
introduced in 
\cite{seide20141},
in which only the large
entries
are transmitted. 
On training a neural machine translation model
with 4 GPUs,
this greatly reduces the communication overhead and achieves 22\% speedup
\cite{aji2017sparse}.
Deep gradient compression \cite{lin2017deep} is another heuristic method that combines gradient sparsification with other techniques such as momentum correction,
local gradient clipping, and momentum factor masking, achieving significant 
reduction on communication cost.
Recently, a stochastic sparsification method was proposed in \cite{wangni2018gradient} that 
balances sparsity and variance
by solving a constrained linear programming. 
MEM-SGD \cite{stich2018sparsified} 
combines top-$k$ sparsification with error correction.
By keeping track of the accumulated
errors, these can be added back to the gradient estimator before each transmission. 
MEM-SGD converges at the same rate as SGD on convex problems, whilst 
reducing the communication overhead by a factor equal to the problem
dimensionality.

On the other hand, 
gradient quantization mitigates the communication bottleneck
by lowering the gradient's floating-point precision with a smaller bit width. 
1-bit SGD achieves state-of-the-art results
on acoustic modeling
while dramatically reducing the communication cost 
\cite{seide20141,strom2015scalable}.
TernGrad \cite{wen2017terngrad} 
quantizes the gradients to ternary levels $\{-1, 0, 1\}$. 
QSGD \cite{alistarh2017qsgd} employs stochastic randomized rounding 
to ensure unbiasedness of the estimator. 
Error-compensated quantized SGD (ECQ-SGD) was proposed in \cite{wu2018error}, wherein a similar stochastic quantization function used in
QSGD is employed, and an error bound is obtained for quadratic loss functions. 
Different from the error-feedback mechanism proposed in MEM-SGD, ECQ-SGD requires
two more hyper-parameters and its quantization errors are decayed exponentially. Thus, error feedback is limited to a small number of iterations. 
Also, ECQ-SGD uses 
all-to-all broadcast (which may involve large network traffic and idle time), while we consider parameter-server architecture. 
Recently, Bernstein {\em et al.\/}
proposed
signSGD with majority vote \cite{pmlr-v80-bernstein18a},
which only 
transmits the 
1-bit 
gradient sign 
between workers and server. 
A variant 
using momentum, called
signum with majority vote, 
is also introduced though without convergence analysis \cite{bernstein2019signsgd} .
Using the majority vote,
signSGD 
achieves a notion of Byzantine fault tolerance \cite{bernstein2019signsgd}.
Moreover, it 
converges at the same rate as
distributed SGD,
though it has to rely on the unrealistic
assumptions of having a 
large mini-batch and 
unimodal symmetric
gradient noise.
Indeed,
signSGD can diverge in some simple cases  when these assumptions are violated
\cite{karimireddy2019error}. 
With only a single worker,
this 
divergence issue
can be fixed by 
using the error correction technique in MEM-SGD, leading to 
SGD with error-feedback (EF-SGD)
\cite{karimireddy2019error}. 

While only a single worker is considered in EF-SGD,
we study in this paper the more interesting distributed setting.
An extension of MEM-SGD and EF-SGD with parallel computing was proposed in
\cite{cordonnier2018convex} for all-to-all broadcast.  
Another related architecture is allreduce. 
Compression at the server can be implemented between the reduce and broadcast 
steps in tree allreduce, or between the reduce-scatter and allgather steps in ring allreduce. 
However, allreduce requires repeated gradient aggregations, and the compressed gradients need to be first
decompressed before they are summed. Hence, heavy overheads may be incurred. 

In this paper,
we study 
the distributed setting with a parameter server architecture. 
To ensure
efficient 
communication,
we consider two-way
gradient compression, in which gradients in both directions (server to/from workers) are compressed.  Note that existing works (except signSGD/signum with majority vote \cite{pmlr-v80-bernstein18a,bernstein2019signsgd}) do not compress the aggregated gradients before sending 
back to
workers. 
Moreover,
as gradients in a deep network typically have similar magnitudes in each
layer,
each layer-wise gradient can be sufficiently represented using a sign vector and
its average $\ell_1$-norm. 
This layer-wise (or blockwise in general) compressor
achieves nearly $32$x reduction in communication cost.
The resulant procedure is called communication-efficient distributed SGD with
error-feedback
(dist-EF-SGD).
Analogous to SGDM, we also propose a stochastic variant 
dist-EF-SGDM
with Nesterov's momentum
\cite{nesterov1983method}.
The convergence properties of 
dist-EF-SGD(M) 
are studied theoretically. 

Our contributions are: (i)
We provide a bound 
on dist-EF-SGD
with general stepsize schedule for a class of compressors
(including the commonly used sign-operator and top-$k$ sparsification). 
In particular, without relying on the unrealistic assumptions in
\cite{pmlr-v80-bernstein18a,bernstein2019signsgd}, we show that dist-EF-SGD 
with constant/decreasing/increasing stepsize
converges at an 
$\mathcal{O}(1/\sqrt{MT})$
rate, which matches that of
distributed synchronous SGD;
(ii) We study gradient compression with Nesterov's momentum in a parameter server. 
For dist-EF-SGDM
with constant stepsize,
we obtain an 
$\mathcal{O}(1/\sqrt{MT})$ rate.
To the best of our knowledge, these are the first convergence results on two-way gradient compression with Nesterov's momentum;
(iii) We propose a general blockwise compressor and show its theoretical properties. 
Experimental results show that
the proposed algorithms 
are efficient
without losing prediction accuracy. 
After our paper has appeared, 
we note 
a similar idea was independently proposed in \cite{tang2019doublesqueeze}. Different from ours, they do not consider changing stepsize, blockwise compressor and Nesterov's momentum.  



{\bf Notations}. For a vector $x$, $\|x\|_1$ and $\|x\|_2$ are its $\ell_1$- and $\ell_2$-norms, respectively. 
$\text{sign}(x)$ outputs a vector in which each element is the sign of the corresponding entry of $x$.  
For two vectors $x, y$, $\langle x, y \rangle$ denotes the dot product.  
For a function $f$, its gradient is $\nabla f$.

\section{Related Work:
SGD with Error-Feedback}


In machine learning, one is often interested in minimizing the expected
risk
$F(x) = \CE_{\xi}[f(x, \xi)]$.
which directly measures the generalization error \cite{bottou2004large}.
Here,
$x \in \R^d$ is the model parameter, 
$\xi$ is drawn from some unknown distribution,
and
$f(x, \xi)$ is the 
possibly nonconvex
risk due to $x$. 
When the expectation is taken over a training set of size $n$, the expected risk
reduces to empirical risk.


Recently, Karimireddy {\em et al.\/} 
\cite{karimireddy2019error} introduced
SGD with error-feedback (EF-SGD),
which combines gradient compression with error correction 
(Algorithm~\ref{alg:ef-sgd}).
A single machine 
is considered,
which keeps the gradient difference 
that is not used for parameter update in the current iteration.
In the next iteration $t$, the accumulated residual
$e_t$ 
is added to the current gradient. The corrected gradient 
$p_t$ 
is then fed into 
an $\delta$-approximate compressor. 

\begin{definition} \label{definition:ef_compressor}
\cite{karimireddy2019error} 
An operator $\mC: \R^d \rightarrow \R^d$ is an $\delta$-approximate compressor for $\delta \in (0, 1]$ if
$\|\mC(x) - x\|^2_2 \leq (1 - \delta)\|x\|_2^2$.
\end{definition}
Examples 
of $\delta$-approximate compressors
include the scaled sign operator
$\mC(v) = \|v\|_1/d\cdot\text{sign}(v)$ \cite{karimireddy2019error} 
and top-$k$ operator (which only preserves the $k$ coordinates with the largest absolute values) \cite{stich2018sparsified}.   
One can also have randomized compressors that only satisfy 
Definition~\ref{definition:ef_compressor}
in expectation.   
Obviously, it is desirable to have 
a large 
$\delta$
while achieving low communication cost. 

\begin{algorithm}[ht]
\caption{SGD with Error-Feedback (EF-SGD) \cite{karimireddy2019error}}
\label{alg:ef-sgd}
\begin{algorithmic}[1]
    \STATE {\bfseries Input:} stepsize $\eta$; compressor $\mathcal{C}(\cdot)$.
    \STATE {\bfseries Initialize:} $x_0 \in \mathbb{R}^d$; $e_{0} = 0 \in \mathbb{R}^d$
    \FOR{$t = 0, \ldots, T-1$}
        \STATE $p_{t} = \eta g_{t} +  e_{t}$ \algorithmiccomment{stochastic gradient $g_{t} = \nabla f(x_t, \xi_t)$}
        \STATE $\Delta_{t} = \mathcal{C}(p_{t})$
		  \COMMENT{compressed value output}
        \STATE $x_{t+1} = x_t - \Delta_{t}$
        \STATE $e_{t+1} = p_{t} - \Delta_{t}$
    \ENDFOR
\end{algorithmic}
\end{algorithm}

EF-SGD achieves the same $\mathcal{O}(1/\sqrt{T}))$ rate as SGD. 
To obtain this convergence guarantee,
an important observation
is that the error-corrected iterate $\tx_t = x_t - e_t$ satisfies the recurrence:
$\tx_{t+1} = \tx_t -  \eta g_t$,
which is similar to that of SGD.
This allows utilizing the convergence proof of SGD to bound the gradient difference $\|\nabla F(\tx_t) - \nabla F(x_t)\|_2$.

\section{Distributed Blockwise Momentum SGD with Error-Feedback}


\subsection{Distributed SGD with Error-Feedback}
\label{sec:dist-ef-sgd}

The proposed procedure,
which extends EF-SGD to the distributed setting.
is shown in Algorithm~\ref{alg:dist-ef-sgd}. 
The computational workload
is distributed
over $M$ workers. A local accumulated error vector $e_{t, i}$ and a local corrected gradient vector $p_{t,i}$ 
are stored in the memory of worker $i$.
At iteration $t$, worker $i$ pushes the compressed signal $\Delta_{t,i} = \mC(p_{t, i})$ to the parameter server. 
On the server side, all workers' $\Delta_{t,i}$'s are aggregated and used to update
its global error-corrected vector $\tp_t$. Before sending back the final update
direction $\tp_{t}$ to each worker, 
compression 
is performed 
to ensure a comparable amount of communication costs between the push and pull operations. 
Due to gradient compression on the server, we also employ a global accumulated error vector $\tilde{e}_t$.  
Unlike EF-SGD in Algorithm~\ref{alg:ef-sgd}, we do not multiply gradient $g_{t, i}$
by the stepsize $\eta_t$ before compression. The two cases make no difference when
$\eta_t$ is constant. However, when the stepsize is changing over time, this would
affect convergence.
We also rescale the local accumulated error $e_{t, i}$ by $\eta_{t-1}/\eta_t$.
This modification,
together with the use of error correction on both workers and server,
allows us to obtain Lemma~\ref{lemma:recurrence_ec_iterate}.  
Because of these differences, 
note that dist-EF-SGD does not reduce to EF-SGD when $M=1$. When $\mathcal{C}(\cdot)$ is
the identity mapping, 
dist-EF-SGD reduces to full-precision distributed SGD. 

\begin{algorithm}[ht]
\caption{Distributed SGD with Error-Feedback (dist-EF-SGD)}
\label{alg:dist-ef-sgd}
\begin{algorithmic}[1]
    \STATE {\bfseries Input:}  stepsize sequence $\{\eta_t\}$ with $\eta_{-1} = 0$; number of workers $M$; compressor $\mathcal{C}(\cdot)$.
    \STATE {\bfseries Initialize:} $x_0 \in \mathbb{R}^d$; $e_{0,i} = 0 \in
	 \mathbb{R}^d$ on each worker $i$; $\tilde{e}_{0} = 0 \in \mathbb{R}^d$ on
	 server
    \FOR{$t = 0, \ldots, T-1$}
        \STATE \textbf{on each worker $i$}
        \STATE \qquad $p_{t,i} = g_{t,i} + \frac{\eta_{t-1}}{\eta_t}e_{t,i}$ \algorithmiccomment{stochastic gradient $g_{t,i} = \nabla f(x_t, \xi_{t, i})$}
        \STATE \qquad \textbf{push} $\Delta_{t,i} = \mathcal{C}(p_{t,i})$ \textbf{to server}
        \STATE \qquad $x_{t+1} = x_t - \eta_t\tilde{\Delta}_{t}$ \algorithmiccomment{$\tilde{\Delta}_{t}$ \textbf{is pulled from server}}
        \STATE \qquad $e_{t+1,i} = p_{t,i} - \Delta_{t,i}$

        \STATE \textbf{on server}
        \STATE \qquad \textbf{pull $\Delta_{t,i}$ from each worker $i$} and $\tp_{t} = \frac{1}{M} \sum_{i=1}^M \Delta_{t,i} + \frac{\eta_{t-1}}{\eta_t}\te_{t}$
        \STATE \qquad \textbf{push $\tilde{\Delta}_t = \mathcal{C}(\tp_{t})$ to each worker}
        \STATE \qquad $\tilde{e}_{t+1} = \tp_t - \tDelta_t$
    \ENDFOR
\end{algorithmic}
\end{algorithm}


In the following,
we investigate the convergence of dist-EF-SGD. 
We make the following assumptions, which are common in the stochastic approximation
literature.

\begin{assumption} \label{assumption:ef_smooth}
$F$ is lower-bounded (i.e., $F_* = \inf_{x \in \R^d} F(x) > -\infty$) and
$L$-smooth (i.e.,
$F(x) \leq F(y) + \langle \nabla F(y), x -y \rangle + \frac{L}{2}\|x - y\|_2^2$
for $x, y \in \R^d$).
\end{assumption}

\begin{assumption} \label{assumption:ef_bounded_var}
The stochastic gradient $g_{t, i}$ has bounded variance: $\CE_t\left[\|g_{t, i} - \nabla F(x_t)\|_2^2\right] \leq \sigma^2$.
\end{assumption}

\begin{assumption} \label{assumption:ef_bounded_grad}
The full gradient $\nabla F$ is uniformly bounded: $\|\nabla F(x_t)\|_2^2 \leq \omega^2$.
\end{assumption}
This implies the 
second moment
is bounded, i.e., $\CE_t\left[\|g_{t, i}\|_2^2\right] \leq G^2 \equiv \sigma^2 + \omega^2$.

\begin{lemma} \label{lemma:recurrence_ec_iterate}
Consider the error-corrected iterate
$\tx_t = x_t - \eta_{t-1}\left(\te_t + \frac{1}{M} \sum_{i=1}^M e_{t,i}\right)$,
where $x_t$, $\te_t$, and $e_{t,i}$'s are generated from Algorithm~\ref{alg:dist-ef-sgd}. It satisfies the recurrence:
$\tx_{t+1} = \tx_t -  \eta_t \frac{1}{M} \sum_{i=1}^Mg_{t,i}$.
\end{lemma}
The above Lemma shows that $\tx_t$ is very similar to the distributed SGD iterate except that the stochastic gradients are evaluated at $x_t$ instead of $\tx_t$. 
This connection allows us to utilize the analysis of full-precision distributed SGD. In particular, we have the following Lemma.
\begin{lemma} \label{lemma:wm_error_bound}
$\CE\left[\left\|\te_{t} + \frac{1}{M} \sum_{i=1}^M e_{t,i}\right\|_2^2\right] \leq  \frac{8(1-\delta)G^2}{\delta^2}\left[1 + \frac{16}{\delta^2}\right]$
for any $t \geq 0$.
\end{lemma}
This implies that $\nabla F(\tx_t) \approx \nabla F(x_t)$ by Assumption~\ref{assumption:ef_smooth}. 
Given the above results, we can prove convergence of the proposed method by
utilizing tools used on the full-precision distributed SGD. 

\begin{theorem} \label{theorem:general_distefsgd_convergence}
Suppose that Assumptions~\ref{assumption:ef_smooth}-\ref{assumption:ef_bounded_grad} hold. 
Assume that $0 < \eta_t < 3/(2L)$ for all $t$.
For the $\{x_t\}$
sequence generated from Algorithm~\ref{alg:dist-ef-sgd}, we have
\begin{eqnarray*}
\CE\left[\left\|\nabla F(x_o)\right\|^2_2\right]
&\leq&  \frac{4}{\sum_{k=0}^{T-1}\eta_k\left(3 - 2L\eta_k \right)}[F(x_0) - F_*]  + \frac{2L\sigma^2}{M}\sum_{t=0}^{T-1}\frac{\eta_t^2}{\sum_{k=0}^{T-1}\eta_k\left(3 - 2L\eta_k \right)} \\
&& + \frac{32L^2(1 - \delta)G^2}{\delta^2}\left[1 + \frac{16}{\delta^2}\right]\sum_{t=0}^{T-1}\frac{\eta_t\eta_{t-1}^2}{\sum_{k=0}^{T-1}\eta_k\left(3 - 2L\eta_k \right)},
\end{eqnarray*}
where $o \in \{0, \dots, T-1\}$ is an index such that $P(o = k) = \frac{\eta_k\left(3 - 2L\eta_k \right)}{\sum_{t=0}^{T-1}\eta_t\left(3 - 2L\eta_t \right)}$, $\forall k = 0, \dots, T-1$. 
\end{theorem}
The first term on the RHS
shows decay of the initial value.
The second term 
is related to the variance,
and the proposed algorithm enjoys variance reduction with more workers. 
The last term 
is due to gradient compression.
A large
$\delta$ (less
compression) 
makes this term smaller and thus faster convergence. 
Similar to the results in \cite{karimireddy2019error}, our bound
also holds
for unbiased compressors (e.g., QSGD \cite{alistarh2017qsgd}) of the form $\mathcal{C}(\cdot) = cU(\cdot)$, where $\CE[U(x)] = x$ and
$\CE[\|U(x)\|_2^2]\leq\frac{1}{c}\|x\|_2^2$ for some $0 < c < 1$. Then, $cU(\cdot)$ is a $c$-approximate compressor in expectation.

The following Corollary 
shows that dist-EF-SGD has a convergence rate of
$\mathcal{O}(1/\sqrt{MT})$, leading to a $\mathcal{O}(1/(M\epsilon^4))$ iteration complexity for satisfying $\CE[\|\nabla F(x_o)\|_2^2] \leq \epsilon^2$. 

\begin{corollary} \label{corollary:convergence-dist-ef-sgd}
Let
stepsize
$\eta = \min(\frac{1}{2L}, \frac{\gamma}{\sqrt{T}/\sqrt{M} + (1 - \delta)^{1/3}\left(1/\delta^2 + 16/\delta^4\right)^{1/3}T^{1/3}})$ 
for some $\gamma > 0$.
Then,
\begin{eqnarray*}
\CE[\|\nabla F(x_o)\|_2^2]
 & \leq &  \frac{4L}{T}[F(x_0) - F_*] + \left[\frac{2}{\gamma}[F(x_0) - F_*] +  L\gamma\sigma^2\right]\frac{1}{\sqrt{MT}} \\
&& + \frac{2(1 - \delta)^{1/3}\left[\frac{1}{\gamma}[F(x_0) - F_*] + 8L^2\gamma^2G^2\right]}{\delta^{2/3}T^{2/3}}\left[1 + \frac{16}{\delta^2}\right]^{1/3}.
\end{eqnarray*}
In comparison,
under the same assumptions,
distributed synchronous SGD 
achieves 
\begin{eqnarray*} \label{eq:dist_sgd_rate}
\CE[\|\nabla F(x_o)\|^2_2]
 \leq  \frac{8L}{3T}[F(x_0) - F_*] + \left[\frac{2}{\gamma}[F(x_0) - F_*]  + L\gamma\sigma^2\right]\frac{2}{3\sqrt{MT}}.
\end{eqnarray*}
\end{corollary}
Thus, the convergence rate of
dist-EF-SGD 
matches that of distributed synchronous SGD (with full-precision gradients)
after $T\geq O(1/\delta^2)$ iterations,
even though gradient compression is used.  
Moreover, more
workers
(larger $M$) leads to faster convergence. 
Note that the bound above does not reduce to that of EF-SGD when $M=1$, as we have two-way compression. 
When $M=1$, our bound also differs from Remark~4 in \cite{karimireddy2019error} in that
our last term is
$O((1-\delta)^{1/3}/(\delta^{4/3}T^{2/3}))$, while theirs is
$O((1-\delta)/(\delta^2T))$ (which is for single machine with one-way compression). Ours is worse
by a factor of
$O(T^{1/3}\delta^{2/3}/(1-\delta)^{2/3})$, which is the price to pay for two-way
compression and a linear
speedup of using $M$ workers. 
Moreover, unlike signSGD with majority vote \cite{pmlr-v80-bernstein18a}, we
achieve a convergence rate of $\mathcal{O}(1/\sqrt{MT})$ without assuming 
a large mini-batch size ($=T$)
and 
unimodal symmetric gradient noise.

Theorem~\ref{theorem:general_distefsgd_convergence} 
only requires $0 < \eta_t < 3/(2L)$ for all $t$. This
thus allows the use
of any decreasing, increasing, or hybrid stepsize schedule. 
In particular, we have the following Corollary.

\begin{corollary} \label{corollary:convergence-decrease-dist-ef-sgd}
Let $\eta_t = \frac{\gamma}{((t+1)T)^{1/4}/(\sqrt{M}) + (1 -
\delta)^{1/3}\left(1/\delta^2 + 16/\delta^4\right)^{1/3}T^{1/3}}$ (decreasing stepsize) with $T \geq 16L^4\gamma^4M^2$ or $\eta_t = \frac{\gamma\sqrt{t+1}}{T/\sqrt{M} + (1 - \delta)^{1/3}\left(1/\delta^2 + 16/\delta^4\right)^{1/3}T^{5/6}}$ (increasing stepsize) with $T \geq 4L^2\gamma^2M$. Then, dist-EF-SGD converges to a stationary point at a rate of $\mathcal{O}(1/\sqrt{MT})$. 
\end{corollary}
To the best of our knowledge, this is the first such result for distributed 
compressed SGD
with decreasing/increasing stepsize on nonconvex problems. These two stepsize
schedules can also be used together. For example, one can
use an increasing stepsize at the beginning of training as warm-up, and then a 
decreasing stepsize afterwards. 


\subsection{Blockwise Compressor}
\label{sec:blc_cp}

\begin{algorithm}[t]
\caption{Distributed Blockwise SGD with Error-Feedback (dist-EF-blockSGD)}
\label{alg:dist-bef-sgd}
\begin{algorithmic}[1]
    \STATE \textbf{Input:} stepsize sequence $\{\eta_t\}$ with $\eta_{-1} = 0$; number of workers $M$; block partition $\{\mG_1, \dots, \mG_B\}$.
    \STATE \textbf{Initialize:} $x_0 \in \mathbb{R}^d$; $e_{0,i} = 0 \in \mathbb{R}^d$ on each worker $i$; $\tilde{e}_{0} = 0 \in \mathbb{R}^d$ on server
    \FOR{$t = 0, \ldots, T-1$}
        \STATE \textbf{on each worker $i$}
         \STATE \qquad $p_{t,i} = g_{t,i} +  \frac{\eta_{t-1}}{\eta_t}e_{t,i}$ \algorithmiccomment{stochastic gradient $g_{t,i} = \nabla f(x_t, \xi_{t, i})$}
        \STATE \qquad \textbf{push} $\Delta_{t,i} = \left[\frac{\|p_{t, i, \mG_1}\|_1}{d_1}\text{sign}(p_{t, i, \mG_1}), \dots, \frac{\|p_{t, i, \mG_B}\|_1}{d_B}\text{sign}(p_{t, i, \mG_B})\right]$ \textbf{to server} 
        \STATE \qquad $x_{t+1} = x_t - \eta_t\tilde{\Delta}_{t}$ \algorithmiccomment{$\tilde{\Delta}_{t}$ \textbf{is pulled from server}}
        \STATE \qquad $e_{t+1,i} = p_{t,i} - \Delta_{t,i}$

        \STATE \textbf{on server}
        \STATE \qquad \textbf{pull} $\Delta_{t,i}$ \textbf{from each worker} $i$ and $\tp_{t} = \frac{1}{M} \sum_{i=1}^M \Delta_{t,i} +  \frac{\eta_{t-1}}{\eta_t}\te_{t}$
        \STATE \qquad \textbf{push} $\tilde{\Delta}_t = \left[\frac{\|\tp_{t,  \mG_1}\|_1}{d_1}\text{sign}(\tp_{t, \mG_1}), \dots, \frac{\|\tp_{t,  \mG_B}\|_1}{d_B}\text{sign}(\tp_{t, \mG_B})\right]$ \textbf{to each worker}
        \STATE \qquad $\tilde{e}_{t+1} = \tp_t - \tDelta_t$
    \ENDFOR
\end{algorithmic}
\end{algorithm}

A commonly used compressor
is
\cite{karimireddy2019error}:
\begin{equation} \label{eq:orig}
\mC(v) = \|v\|_1/d \cdot\text{sign}(v).
\end{equation} 
Compared to using only the sign operator as in signSGD, the factor $\|v\|_1/d$ can preserve the gradient's magnitude. 
However, as shown in
\cite{karimireddy2019error}, its
$\delta$ in Definition~\ref{definition:ef_compressor}
is  $\|v\|_1^2/(d\|v\|_2^2)$,
and can be
particularly small when $v$ is sparse. 
When 
$\delta$ is closer to $1$,
the bound 
in Corollary~\ref{corollary:convergence-dist-ef-sgd}
becomes smaller  and thus 
convergence is
faster.
In this section,  we achieve this
by
proposing a blockwise extension of
(\ref{eq:orig}).

Specifically, we 
partition the
compressor input 
$v$
into $B$ blocks, where
each block $b$ has $d_b$ elements indexed by
$\mG_b$.  
Block $b$
is then compressed with 
scaling factor
$\|v_{\mG_b}\|_1/d_b$ (where $v_{\mG_b}$ is the subvector of $v$ with elements in
block $b$),
leading to:
$\mC_B(v) = [\|v_{\mG_1}\|_1/d_1 \cdot \text{sign}(v_{\mG_1}), \dots,
\|v_{\mG_B}\|_1/d_B\cdot\text{sign}(v_{\mG_B})]$. 
A similar 
compression scheme, 
with each layer being a block,
is considered in the experiments 
of \cite{karimireddy2019error}. However, they provide no
theoretical justifications. 
The following Proposition first shows that $\mC_B(\cdot)$ is also an approximate compressor.

 \begin{proposition} \label{prop:block_density}
  Let $[B] = \{1, 2, \dots, B\}$. 
 $\mC_B$ 
 is a
 $\phi(v)$-approximate compressor, where 
 $\phi(v)= \min_{b \in [B]} \frac{\|v_{\mG_b}\|_1^2}{d_b\|v_{\mG_b}\|_2^2} \geq \min_{b \in [B]} \frac{1}{d_b}$.  
 \end{proposition}
The resultant algorithm
will be 
called dist-EF-blockSGD 
(Algorithm~\ref{alg:dist-bef-sgd}) in the sequel. 
As can be seen, this is a special case of Algorithm~\ref{alg:dist-ef-sgd}.
By replacing $\delta$ with $\phi(v)$ in
  Proposition~\ref{prop:block_density}, 
the convergence results of dist-EF-SGD in Section~\ref{sec:dist-ef-sgd} can be directly applied.


There are many ways to partition the gradient into blocks. In practice, one can simply consider 
each parameter
tensor/matrix/vector in the deep network as a block. The intuition is that (i)
gradients in the same parameter tensor/matrix/vector typically have similar
magnitudes,
and (ii) the corresponding scaling factors can thus be tighter than the scaling factor obtained on the whole parameter, leading to a larger $\delta$. 
As an illustration of (i), 
Figure~\ref{fig:cv_b16_cifar100} shows 
the 
coefficient of variation
(which 
is defined as the ratio of the standard deviation to the mean)
of $\{|g_{t, i}|\}_{i \in \mG_b}$ 
averaged over all blocks and iterations in an epoch, 
obtained from ResNet-20 on the CIFAR-100 dataset (with a mini-batch size of 16 per worker).\footnote{The detailed experimental setup is in  Section~\ref{sec:distef-cifar100-expt}.}
A value smaller than $1$ indicates that the absolute gradient values in each block concentrate around the mean.
As for point (ii) above,
consider the case where all the blocks
are of the same size ($d_b = \tilde{d}, \forall b$), elements in the
same block have the same magnitude
($\forall i \in \mathcal{G}_b,
|v_i| = c_b$ for some $c_b$),
and 
the magnitude is increasing across blocks ($c_b/c_{b+1}=\alpha$ for some $\alpha < 1$). 
For the standard compressor in 
(\ref{eq:orig}), 
$\delta=\frac{\|v\|_1^2}{d\|v\|_2^2} = \frac{(1+\alpha)(1-\alpha^B)}{B(1-\alpha)(1+\alpha^B)} \approx
\frac{(1+\alpha)}{B(1-\alpha)}$ for a sufficiently large $B$;
whereas for the proposed blockwise compressor,
$\phi(v) = 1
\gg \frac{(1+\alpha)}{B(1-\alpha)}$.
 Figure~\ref{fig:delta_b16_cifar100}
shows the empirical estimates 
of 
$\|v\|_1^2/(d\|v\|_2^2)$ and $\phi(v)$
in the ResNet-20 experiment.
As can be seen, $\phi(v)
\gg \|v\|_1^2/(d\|v\|_2^2)$.


 

\begin{figure}[ht]
\vskip -.2in
\begin{center}
\subfigure[Coefficient of variation of $\{|g_{t, i}|\}_{i \in \mG_b}$.]{\includegraphics[width=0.4\columnwidth]{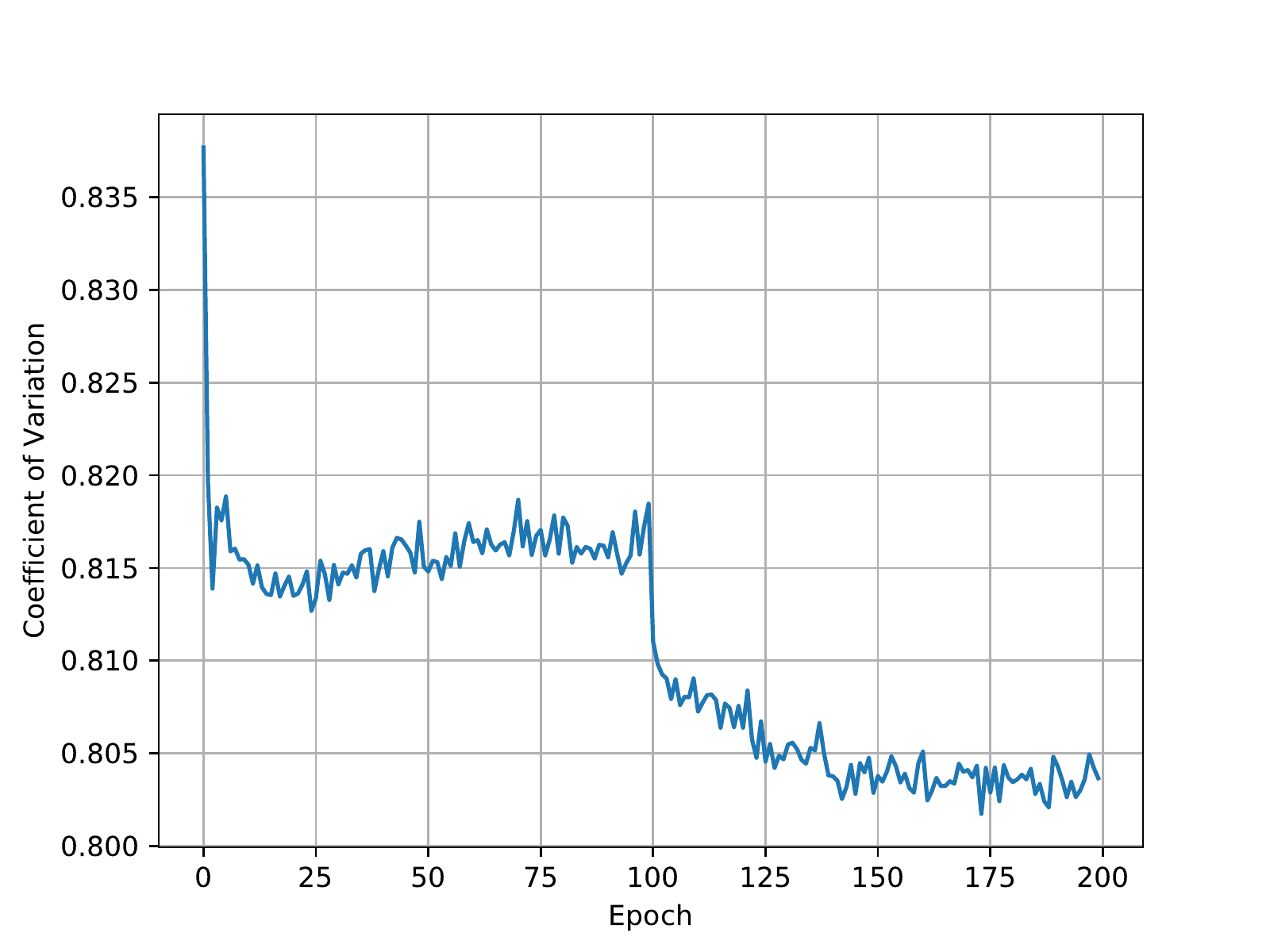} \label{fig:cv_b16_cifar100}}
\subfigure[$\delta$ for blockwise and non-block versions.]{\includegraphics[width=0.4\columnwidth]{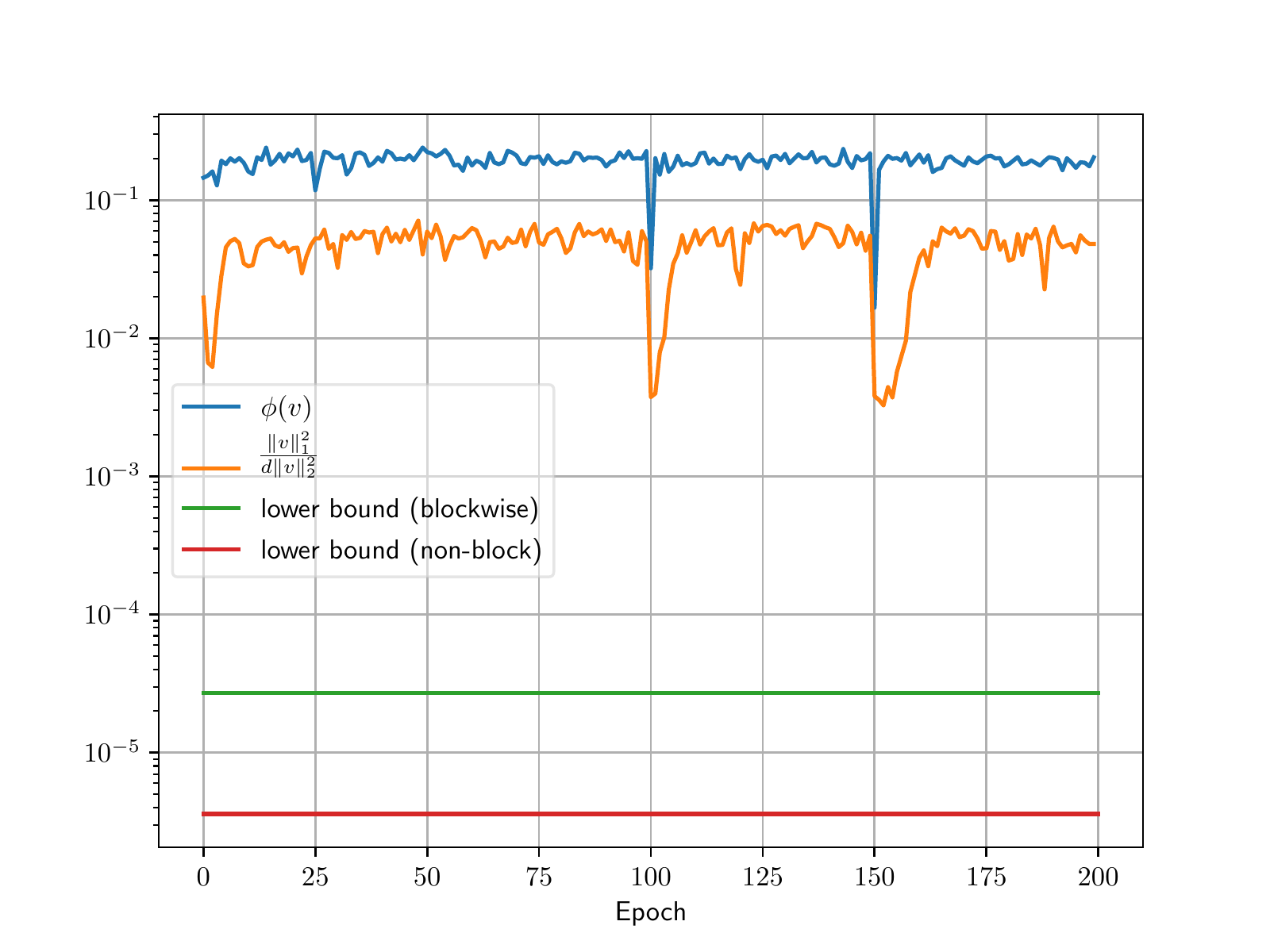} \label{fig:delta_b16_cifar100}}
\vskip -.1in
\caption{Illustrations using the ResNet-20 in 
Section~\ref{sec:distef-cifar100-expt}.
Left: Averaged coefficient of variation of $\{|g_{t, i}|\}_{i \in \mG_b}$. 
Right: Empirical estimates of $\delta$ for the blockwise ($\phi(v)$ in Proposition~\ref{prop:block_density}) and non-block versions ($\|v\|_1^2/(d\|v\|_2^2)$).
Each point is the minimum among all iterations in an epoch. 
The lower bounds,
$\min_{b \in [B]} 1/d_b$ and $1/d$,
are also shown. 
Note that the ordinate is in log scale. 
}
\end{center}
\end{figure}

The 
per-iteration communication costs of the various distributed algorithms 
are shown in
Table~\ref{tbl:bit_comp}.
Compared to signSGD with majority vote \cite{pmlr-v80-bernstein18a}, 
dist-EF-blockSGD requires an extra
$64MB$ bits 
for transmitting the blockwise scaling factors
(each factor $\|v_{\mG_b}\|_1/d_b$ is stored
in float32 format and transmitted twice in each iteration).
By treating each vector/matrix/tensor parameter as a block,
$B$ is typically in the order of hundreds. For most
problems of interest, $64MB/(2Md) < 10^{-3}$.
The reduction in communication cost 
compared to full-precision distributed SGD
is thus nearly 32x.

 
 \begin{table}[h]
 \vskip -.1in
 \begin{center}
 \caption{Communication costs of the various distributed gradient compression
 algorithms and SGD.}
 \vskip .1in
 \begin{tabular}{c|c}
 \hline
 algorithm & \#bits per iteration \\
  \hline
 full-precision SGD & $64Md$ \\
  signSGD with majority vote & $2Md$ \\
  dist-EF-blockSGD & $2Md + 64MB$\\
  \hline
 \end{tabular}
  \label{tbl:bit_comp}
 \end{center}
 \vskip -.2in
 \end{table}

\subsection{Nesterov's Momentum}
\label{sec:dist-ef-sgdm}

Momentum has been widely used in deep networks 
\cite{sutskever2013importance}.
Standard distributed SGD with 
Nesterov's momentum 
\cite{nesterov1983method}
and full-precision gradients
uses the update: $m_{t,i} = \mu m_{t-1,i} + g_{t, i}, \forall i \in [M]$ and $x_{t+1} =x_t -  \eta_t \frac{1}{M} \sum_{i=1}^M(\mu m_{t,i} + g_{t, i})$, 
where 
$m_{t, i}$ 
is a local momentum vector maintained by each worker $i$ at time $t$ (with
$m_{0,i} = 0$), and $\mu \in [0, 1)$ is the momentum parameter.
In this section, we extend
the proposed dist-EF-SGD 
with momentum. 
Instead of sending the compressed $g_{t, i} +
\frac{\eta_{t-1}}{\eta_t}e_{t, i}$
to the server,
the compressed
$\mu m_{t, i} + g_{t, i} + \frac{\eta_{t-1}}{\eta_t}e_{t, i}$ is sent.
The server merges all the workers's results 
and sends it back to each worker.
The resultant procedure with blockwise compressor is called dist-EF-blockSGDM (Algorithm~\ref{alg:dist-bef-sgdm}),
and has the same communication cost as dist-EF-blockSGD. 
The
corresponding non-block variant is analogous.

\vskip -.1in
\begin{algorithm}[ht]
\caption{Distributed Blockwise Momentum SGD with Error-Feedback (dist-EF-blockSGDM)}
\label{alg:dist-bef-sgdm}
\begin{algorithmic}[1]
    \STATE \textbf{Input:} stepsize sequence $\{\eta_t\}$ with $\eta_{-1} = 0$; momentum parameter $0 \leq \mu < 1$; number of workers $M$; block partition $\{\mG_1, \dots, \mG_B\}$.
    \STATE \textbf{Initialize:} $x_0 \in \mathbb{R}^d$; $m_{-1, i} = e_{0,i} = 0 \in \mathbb{R}^d$ on each worker $i$; $\tilde{e}_{0} = 0 \in \mathbb{R}^d$ on server
    \FOR{$t = 0, \ldots, T-1$}
        \STATE \textbf{on each worker $i$}
         \STATE \qquad $m_{t,i} = \mu m_{t - 1,i} +  g_{t,i}$ \algorithmiccomment{stochastic gradient $g_{t,i} = \nabla f(x_t, \xi_{t, i})$}
        \STATE \qquad $p_{t,i} = \mu m_{t,i} + g_{t,i} +  \frac{\eta_{t-1}}{\eta_t}e_{t,i}$
        \STATE \qquad \textbf{push} $\Delta_{t,i} = \left[\frac{\|p_{t, i, \mG_1}\|_1}{d_1}\text{sign}(p_{t, i, \mG_1}), \dots, \frac{\|p_{t, i, \mG_B}\|_1}{d_B}\text{sign}(p_{t, i, \mG_B})\right]$ \textbf{to server} 
        \STATE \qquad $x_{t+1} = x_t - \eta_t\tilde{\Delta}_{t}$ \algorithmiccomment{$\tilde{\Delta}_{t}$ \textbf{is pulled from server}}
        \STATE \qquad $e_{t+1,i} = p_{t,i} - \Delta_{t,i}$

        \STATE \textbf{on server}
        \STATE \qquad \textbf{pull} $\Delta_{t,i}$ \textbf{from each worker} $i$ and $\tp_{t} = \frac{1}{M} \sum_{i=1}^M \Delta_{t,i} +  \frac{\eta_{t-1}}{\eta_t}\te_{t}$
        \STATE \qquad \textbf{push} $\tilde{\Delta}_t = \left[\frac{\|\tp_{t,  \mG_1}\|_1}{d_1}\text{sign}(\tp_{t, \mG_1}), \dots, \frac{\|\tp_{t,  \mG_B}\|_1}{d_B}\text{sign}(\tp_{t, \mG_B})\right]$ \textbf{to each worker}
        \STATE \qquad $\tilde{e}_{t+1} = \tp_t - \tDelta_t$
    \ENDFOR
\end{algorithmic}
\end{algorithm}
\vskip -.1in

Similar to Lemma~\ref{lemma:recurrence_ec_iterate}, the following Lemma shows that
the error-corrected iterate $\tx_t$ is very similar to Nesterov's accelerated
gradient iterate, except that the momentum is computed based on $\{x_t\}$. 
\begin{lemma} \label{lemma:recurrence_mec_iterate}
The error-corrected iterate $\tx_t = x_t - \eta_{t-1}(\te_t + \frac{1}{M} \sum_{i=1}^M e_{t,i})$, where 
$x_t$, $\te_t$, and $e_{t,i}$'s are generated from 
Algorithm~\ref{alg:dist-bef-sgdm},
satisfies the recurrence:
$\tx_{t+1} = \tx_t -  \eta_t \frac{1}{M} \sum_{i=1}^M(\mu m_{t,i} + g_{t, i})$.
\end{lemma}
As in Section~\ref{sec:dist-ef-sgd},
it can be shown that 
$\|\te_t + \frac{1}{M} \sum_{i=1}^M e_{t,i}\|_2$ is bounded and 
$\nabla F(\tx_t) \approx \nabla F(x_t)$. 
The following Theorem shows 
the convergence rate of the proposed 
dist-EF-blockSGDM.

 \begin{theorem} \label{theorem:general_distefsgdm_convergence}
Suppose that Assumptions~\ref{assumption:ef_smooth}-\ref{assumption:ef_bounded_grad} hold. 
Let $\eta_t = \eta$ for some $\eta > 0$. 
For any $\eta \leq \frac{(1 - \mu)^2}{2L}$, and the $\{x_t\}$ sequence generated from Algorithm~\ref{alg:dist-bef-sgdm}, we have
\begin{eqnarray}
\CE\left[\left\|\nabla F(x_o)\right\|^2_2\right]
& \leq &  \frac{4(1 - \mu)}{\eta T}[F(x_0) - F_*] 
+ \frac{2L\eta\sigma^2}{(1 - \mu)M}\left[1 + \frac{2L\eta\mu^4}{(1 - \mu)^3}\right] \label{bound:general_distefsgdm_convergence} \\
&& + \frac{32L^2\eta^2(1-\delta)G^2}{\delta^2(1-\mu)^2}\left[1 + \frac{16}{\delta^2}\right].  \nonumber
\end{eqnarray}
\end{theorem}
Compared to Theorem~\ref{theorem:general_distefsgd_convergence}, using a larger momentum parameter
$\mu$ makes the first term (which depends on the initial condition) smaller but 
a worse variance 
term (second term) and error term due to gradient compression
(last term).
Similar to Theorem~\ref{theorem:general_distefsgd_convergence}, a larger $\eta$ makes the third term larger. 
The following Corollary shows that the proposed dist-EF-blockSGDM achieves a
convergence rate of $\mathcal{O}(((1 - \mu)[F(x_0) - F_*] + \sigma^2/(1-\mu))/\sqrt{MT})$. 

\begin{corollary} \label{corollary:convergence-dist-ef-sgdm}
Let $\eta = \frac{\gamma}{\sqrt{T}/\sqrt{M} + (1 - \delta)^{1/3}\left(1/\delta^2 + 16/\delta^4\right)^{1/3}T^{1/3}}$ for some $\gamma > 0$. 
For any $T \geq \frac{4\gamma^2L^2M}{(1 - \mu)^4}$,
$\CE\left[\left\|\nabla F(x_o)\right\|^2_2\right]
\leq   \left[\frac{2(1 - \mu)}{\gamma}[F(x_0) - F_*]  + \frac{L\gamma\sigma^2}{1-\mu}\right]\frac{2}{\sqrt{MT}} 
+ \frac{4L^2\gamma^2\mu^4\sigma^2}{(1 - \mu)^4T} 
+ \frac{4(1 - \delta)^{1/3}\left[\frac{(1 - \mu)}{\gamma}[F(x_0) - F_*] + 8L^2\gamma^2G^2/(1-\mu)^2\right]}{\delta^{2/3}T^{2/3}}\left[1 + \frac{16}{\delta^2}\right]^{1/3}$.
\end{corollary}

 \section{Experiments}
 

 
\subsection{Multi-GPU Experiment on CIFAR-100}
\label{sec:distef-cifar100-expt}

In this experiment, we demonstrate that the proposed 
 dist-EF-blockSGDM 
 and dist-EF-blockSGD ($\mu = 0$ in Algorithm~\ref{alg:dist-bef-sgdm}), 
though using fewer bits for gradient transmission,
still has good convergence.
For faster experimentation,
we use a 
single node 
with multiple GPUs (an AWS P3.16 instance with 8 Nvidia V100 GPUs,
each GPU being a worker) 
instead of a distributed setting.

Experiment is performed on 
the CIFAR-100 dataset, with
50K training images and 10K test images.
We use a 20-layer ResNet \cite{he2016identity}.
Each parameter tensor/matrix/vector is treated as a block
in dist-EF-blockSGD(M). 
They are compared with
 (i) distributed synchronous SGD (with full-precision gradient);
 (ii) distributed synchronous SGD 
 (full-precision gradient)
 with momentum (SGDM); 
 (iii) signSGD with majority vote \cite{pmlr-v80-bernstein18a}; and
 (iv) signum with majority vote \cite{bernstein2019signsgd}.
All the algorithms are implemented in MXNet. 
We vary the mini-batch size per worker in $\{8, 16, 32\}$. 
Results are averaged over 5 repetitions. 
More details of the experiments are shown
in 
Appendix~\ref{app:cifar}. 

Figure~\ref{fig:ef_cifar100} shows
convergence of the testing accuracy w.r.t.
the number of epochs. 
As can be seen, 
dist-EF-blockSGD converges as fast as SGD and has slightly better accuracy, while 
signSGD performs poorly. In particular, dist-EF-blockSGD is robust to the mini-batch size, while the performance of signSGD 
degrades with smaller mini-batch size (which agrees with the results in
\cite{pmlr-v80-bernstein18a}).
Momentum makes SGD and dist-EF-blockSGD faster with mini-batch size of $16$ or $32$ per worker, 
particularly before epoch $100$.
At epoch 100, the learning rate is reduced, and the difference is less obvious. 
This is because a larger 
mini-batch 
means smaller 
variance
$\sigma^2$, 
so the initial optimality gap 
$F(x_0) - F_*$ in (\ref{bound:general_distefsgdm_convergence}) 
is more dominant.
Use of momentum 
($\mu > 0$) is then beneficial.
On the other hand, momentum significantly
improves signSGD. 
However, signum is still much worse than dist-EF-blockSGDM. 

 \begin{figure}[t]
\begin{center}
\subfigure{\includegraphics[width=0.32\columnwidth]{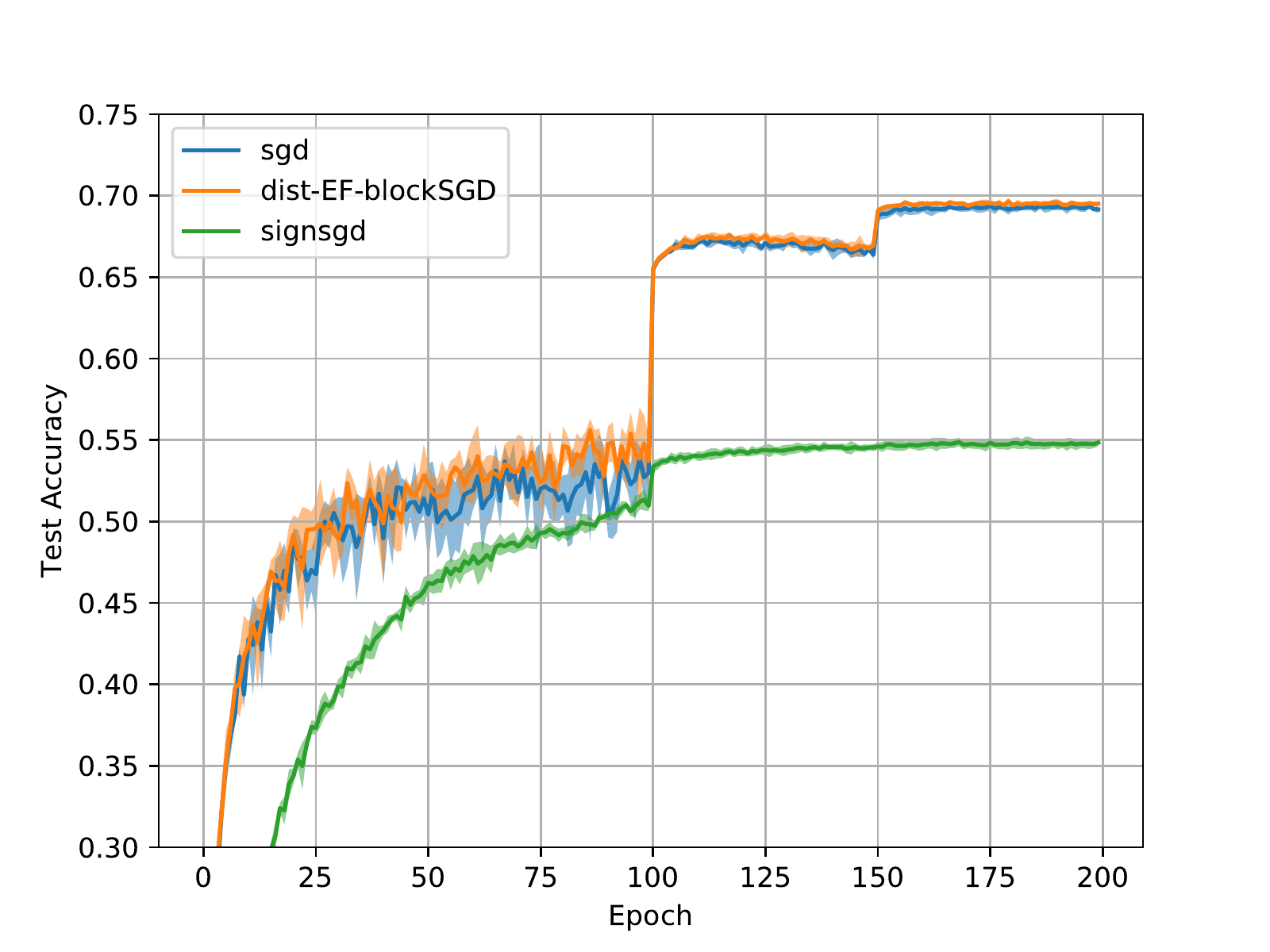}}
\subfigure{\includegraphics[width=0.32\columnwidth]{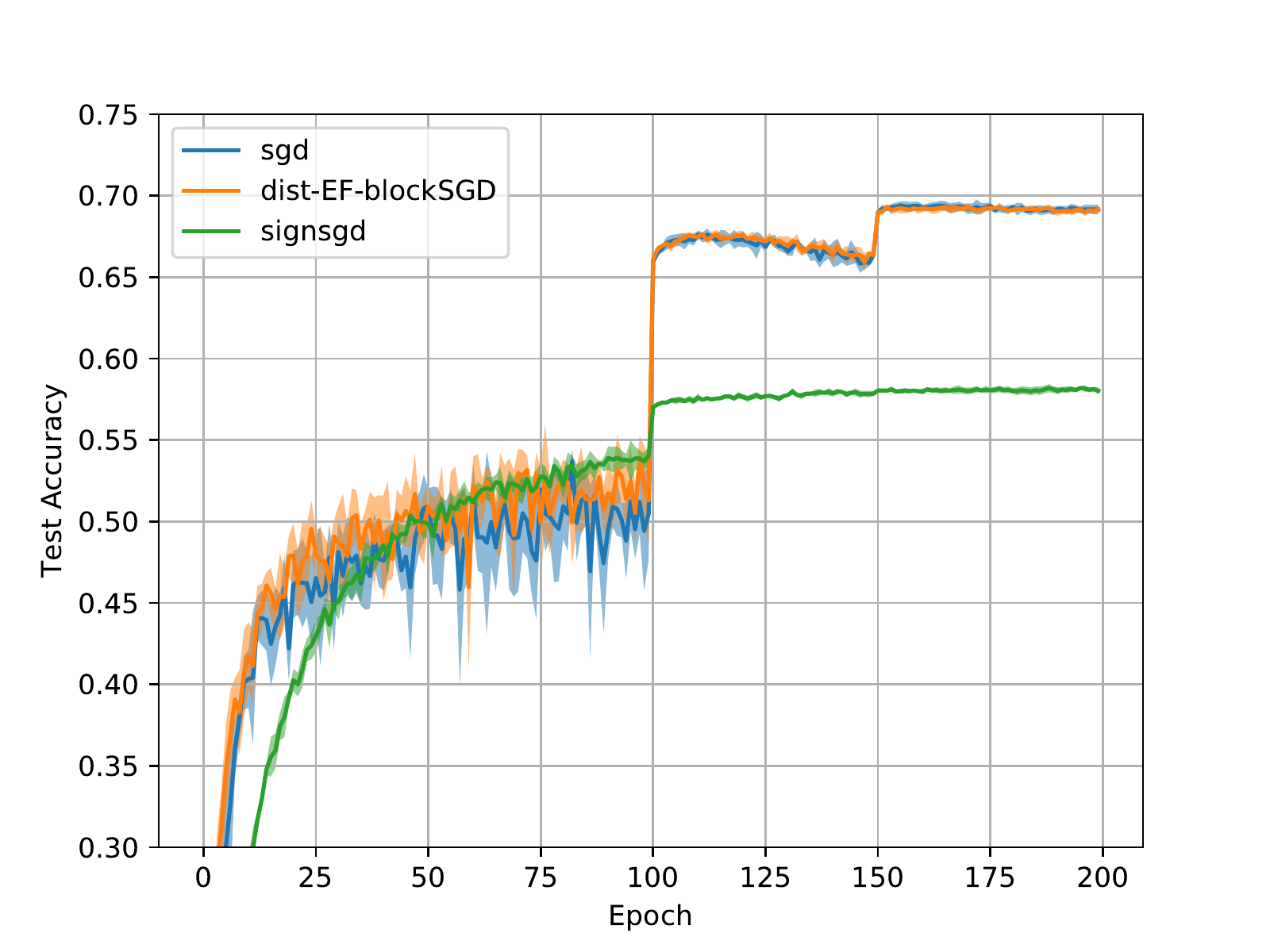}}
\subfigure{\includegraphics[width=0.32\columnwidth]{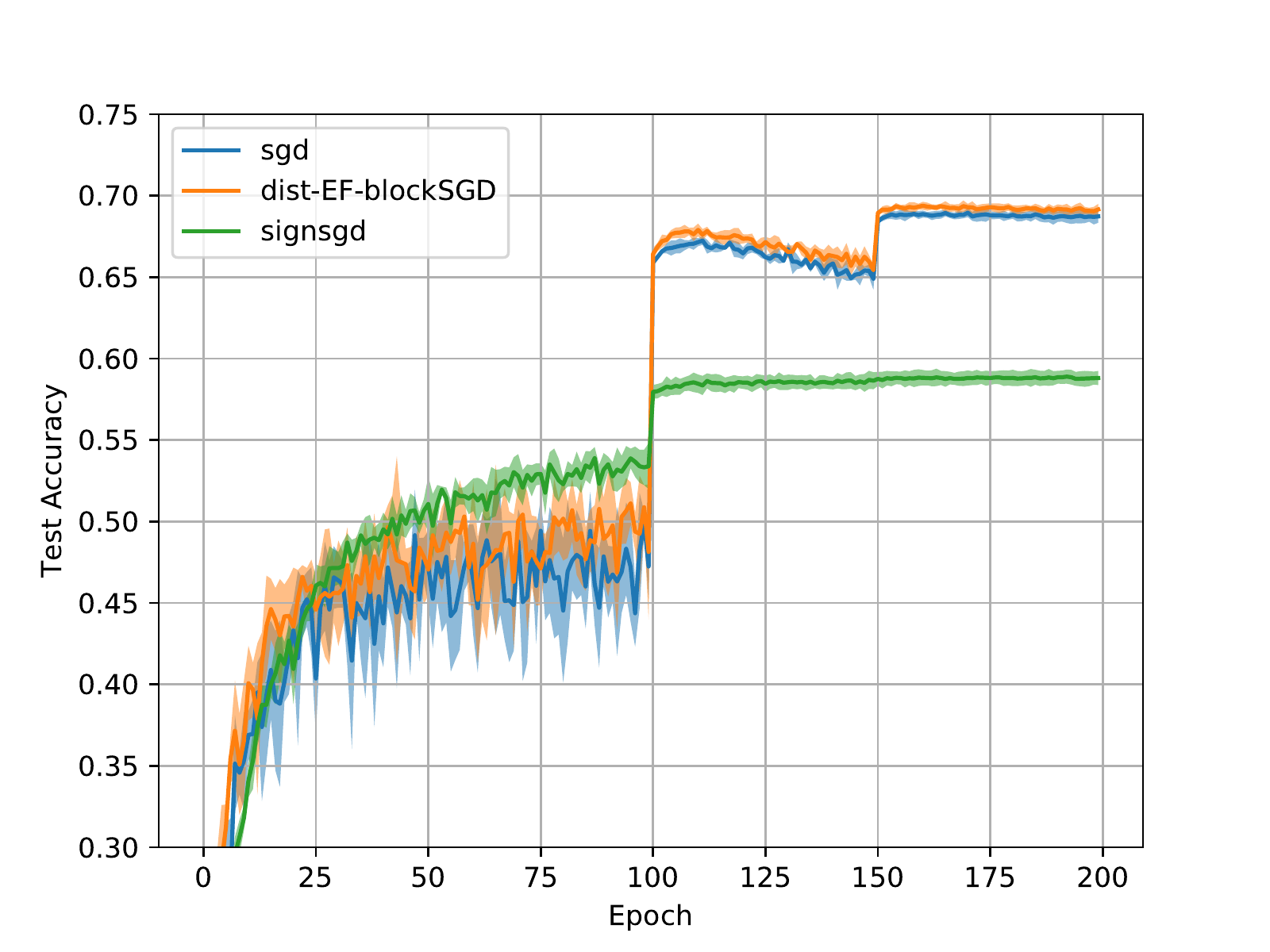}} \\
\setcounter{subfigure}{0}
\vskip -.18in
\subfigure[Mini-batch size: $8$ per worker.]{\includegraphics[width=0.32\columnwidth]{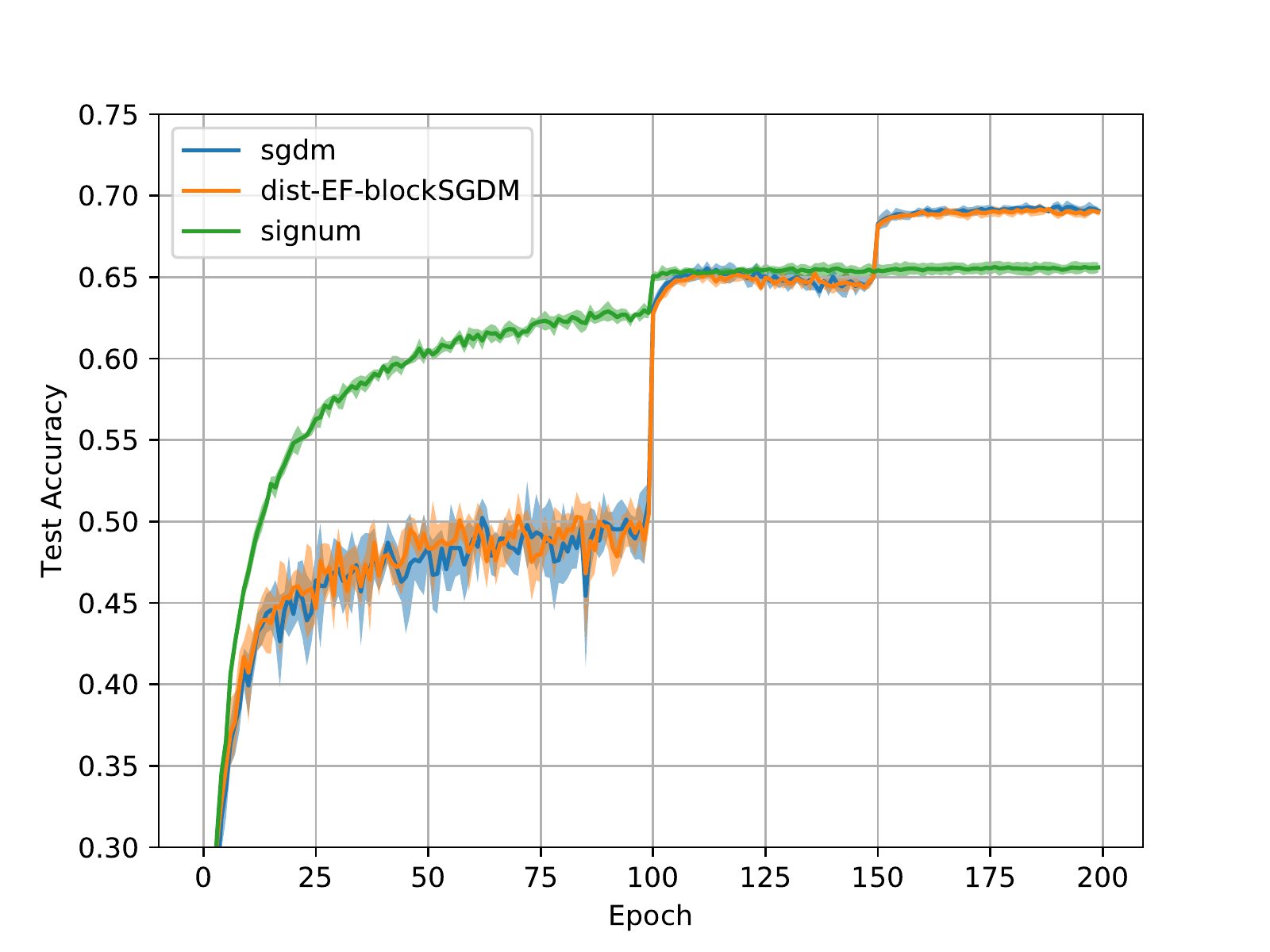}}
\subfigure[Mini-batch size: $16$ per worker.]{\includegraphics[width=0.32\columnwidth]{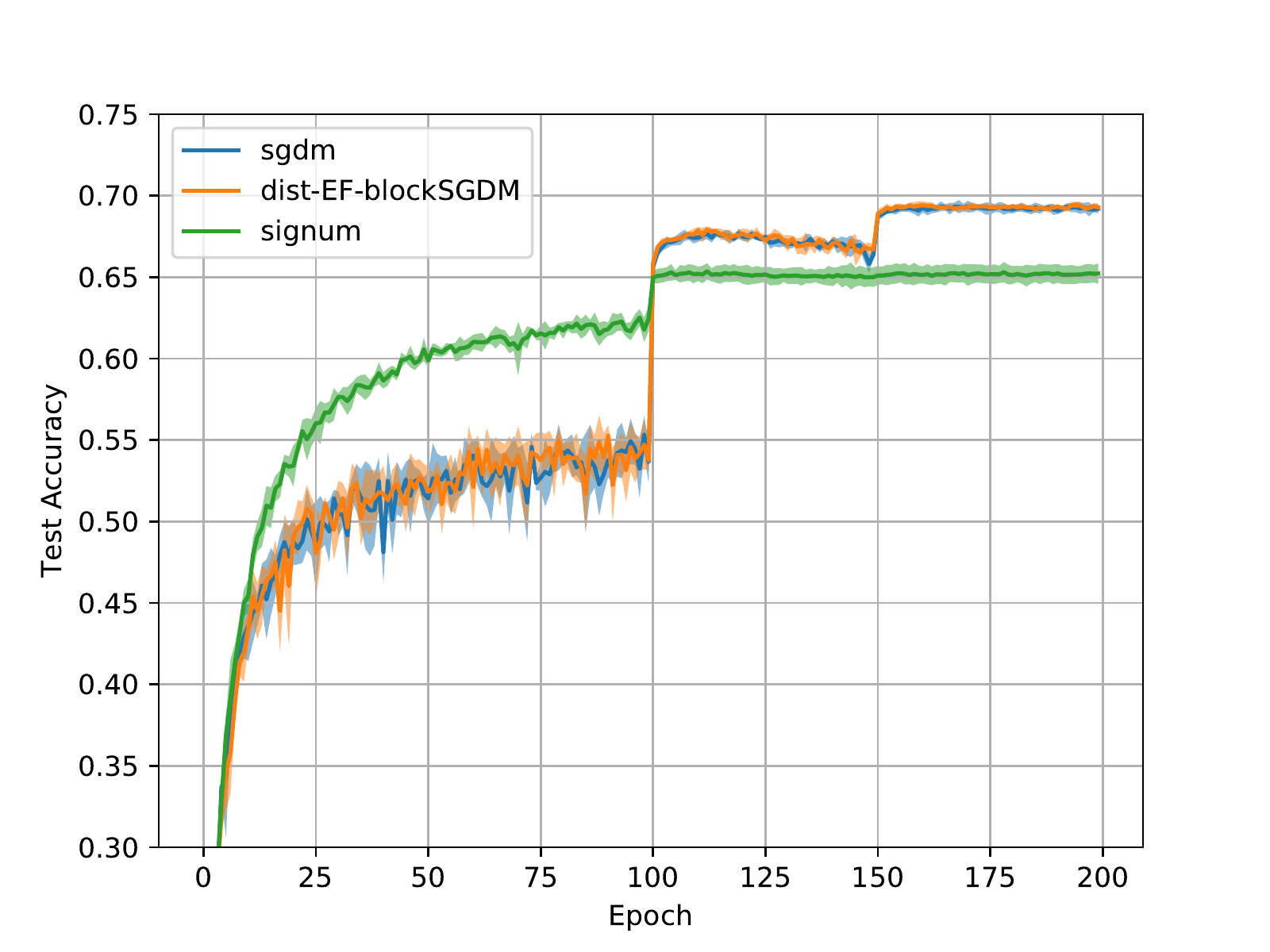}}
\subfigure[Mini-batch size: $32$ per worker.]{\includegraphics[width=0.32\columnwidth]{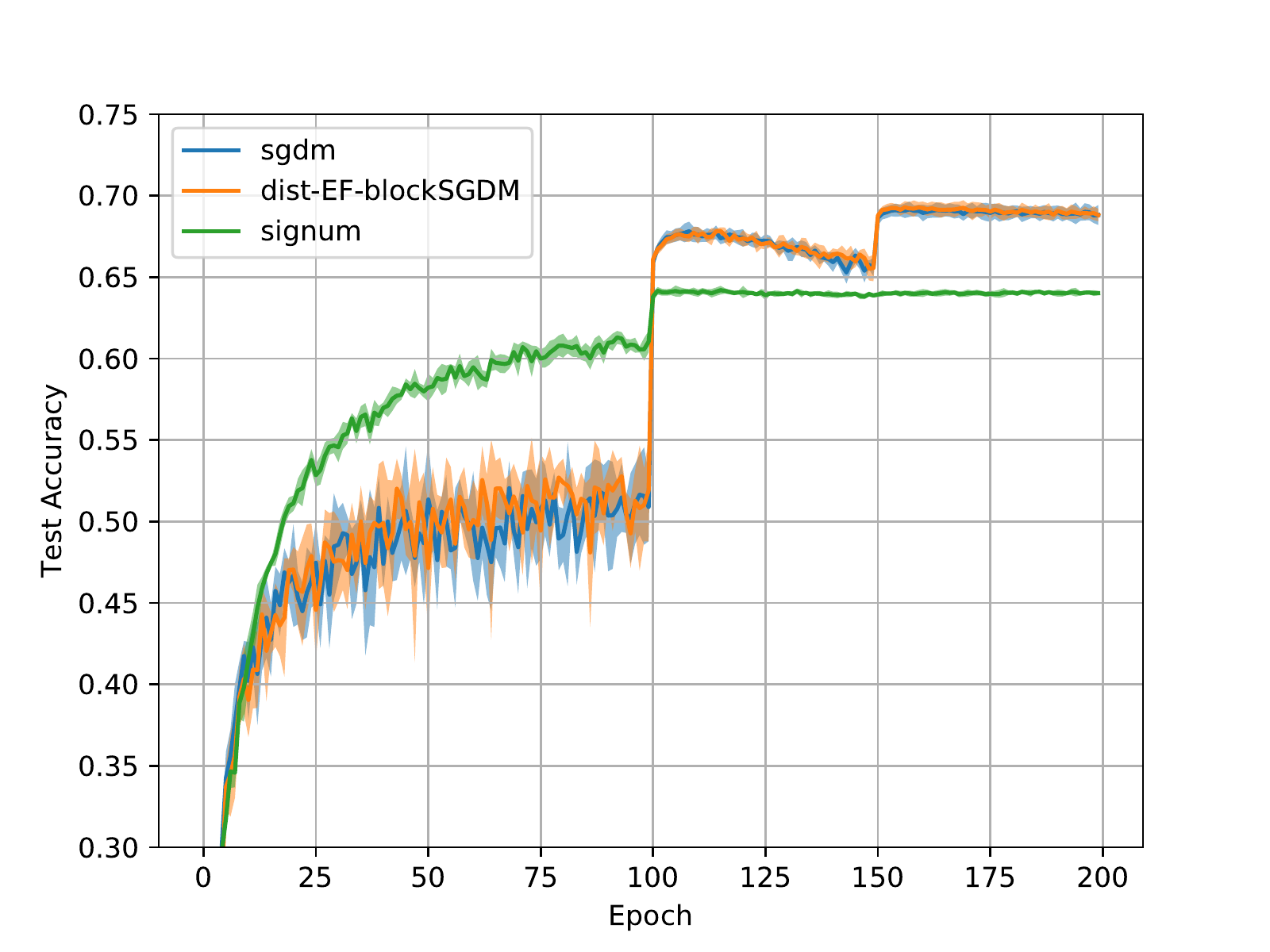}}
\vskip -.1in
\caption{Testing accuracy on CIFAR-100. Top: No momentum; Bottom: With
momentum. The solid curve is the mean accuracy
over five repetitions. The shaded region
spans one standard deviation.}
\label{fig:ef_cifar100}
\end{center}
\vskip -.3in
\end{figure}

\subsection{Distributed Training on ImageNet}
\label{sec:distef-imagenet-expt}

In this section, we perform distributed optimization on ImageNet \cite{russakovsky2015imagenet} 
using
a 50-layer ResNet.
Each worker is an AWS P3.2 instance with 1 GPU,
and the parameter server is housed in one node. 
We use the publicly available
code\footnote{\url{https://github.com/PermiJW/signSGD-with-Majority-Vote}} in
\cite{bernstein2019signsgd}, and the default communication library
Gloo communication
library in PyTorch. 
As in 
\cite{bernstein2019signsgd}, 
we use its allreduce implementation 
for SGDM, which is faster.

As momentum accelerates the training for large mini-batch size in Section~\ref{sec:distef-cifar100-expt}, 
we only compare the momentum variants here.
 The proposed dist-EF-blockSGDM is 
 compared with
(i) distributed synchronous SGD with momentum (SGDM); and
 (ii) signum with majority vote \cite{bernstein2019signsgd}.   
The number of workers $M$ is varied in $\{7, 15\}$.  
With an odd number of workers, a majority vote will not
produce zero, and so signum does not lose accuracy by using 1-bit compression.
More details of the setup are in 
Appendix~\ref{app:imagenet}.

\begin{figure}[t]
\begin{center}
\subfigure[Test accuracy w.r.t. epoch.]{\includegraphics[width=0.32\columnwidth]{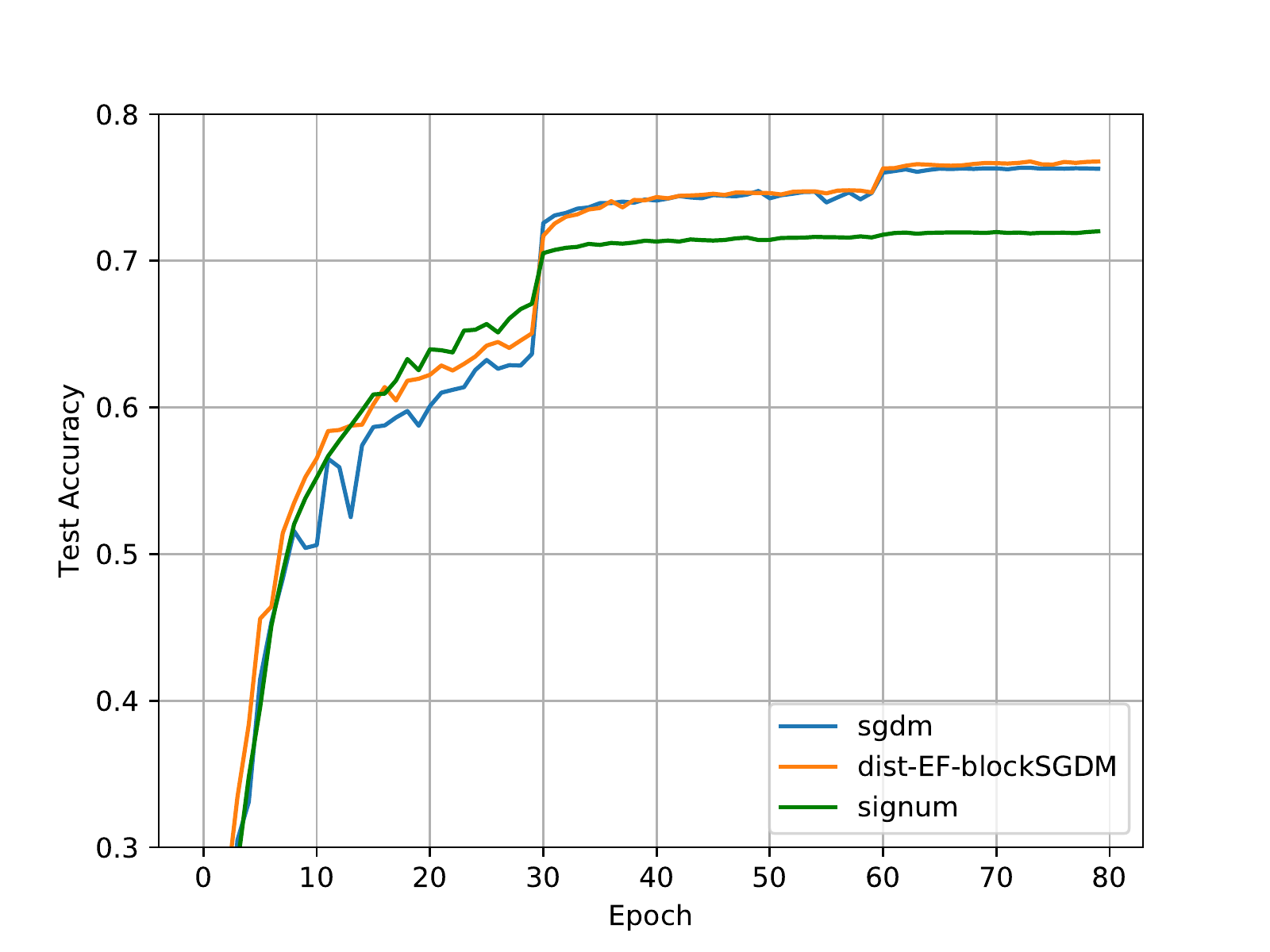}}
\subfigure[Test accuracy w.r.t. time. \label{fig:t1}]{\includegraphics[width=0.32\columnwidth]{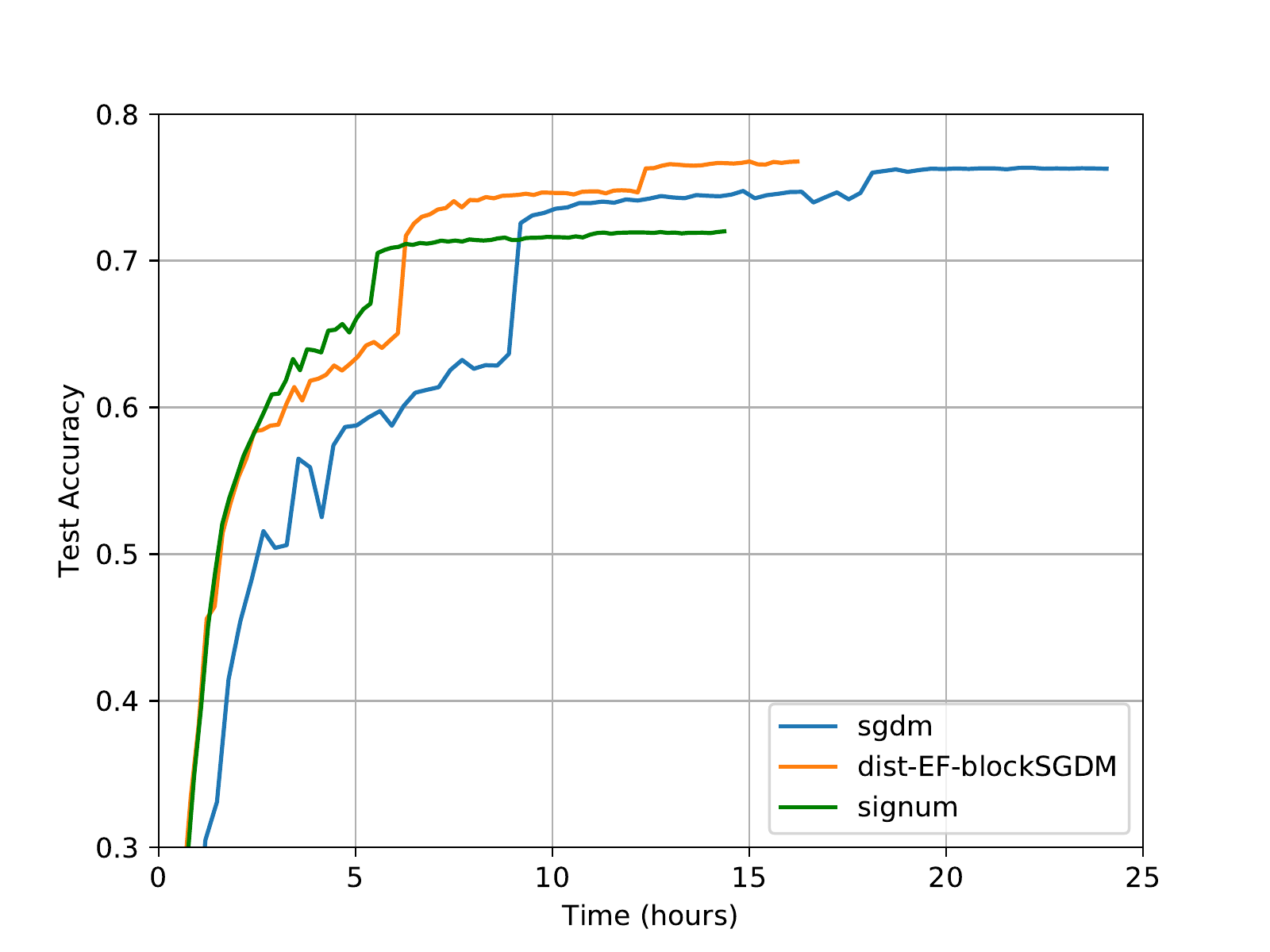}} 
\subfigure[Workload breakdown. \label{fig:load1}]{\includegraphics[width=0.32\columnwidth]{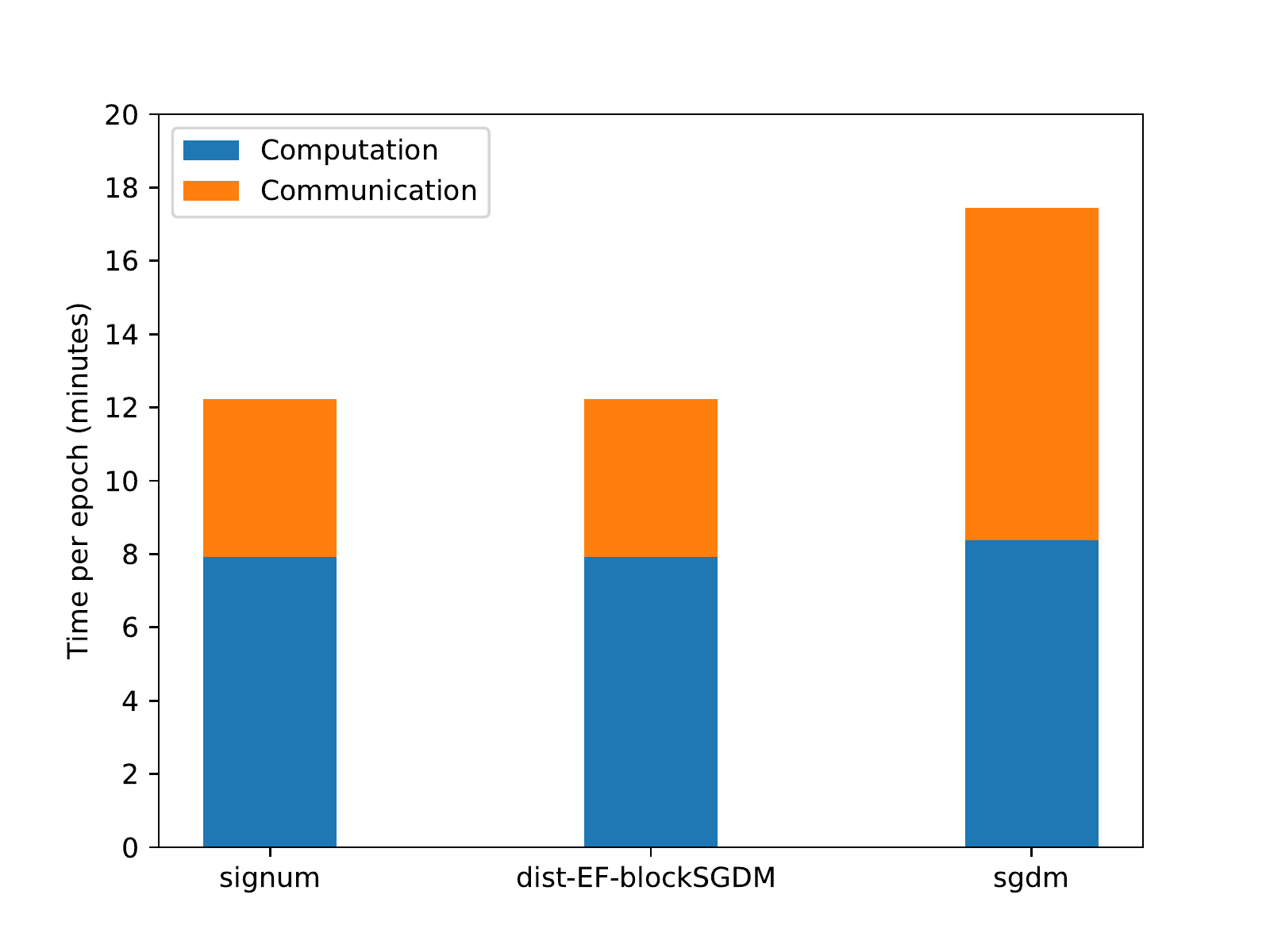}}
\vskip -.15in
\subfigure[Test accuracy w.r.t. epoch.]{\includegraphics[width=0.32\columnwidth]{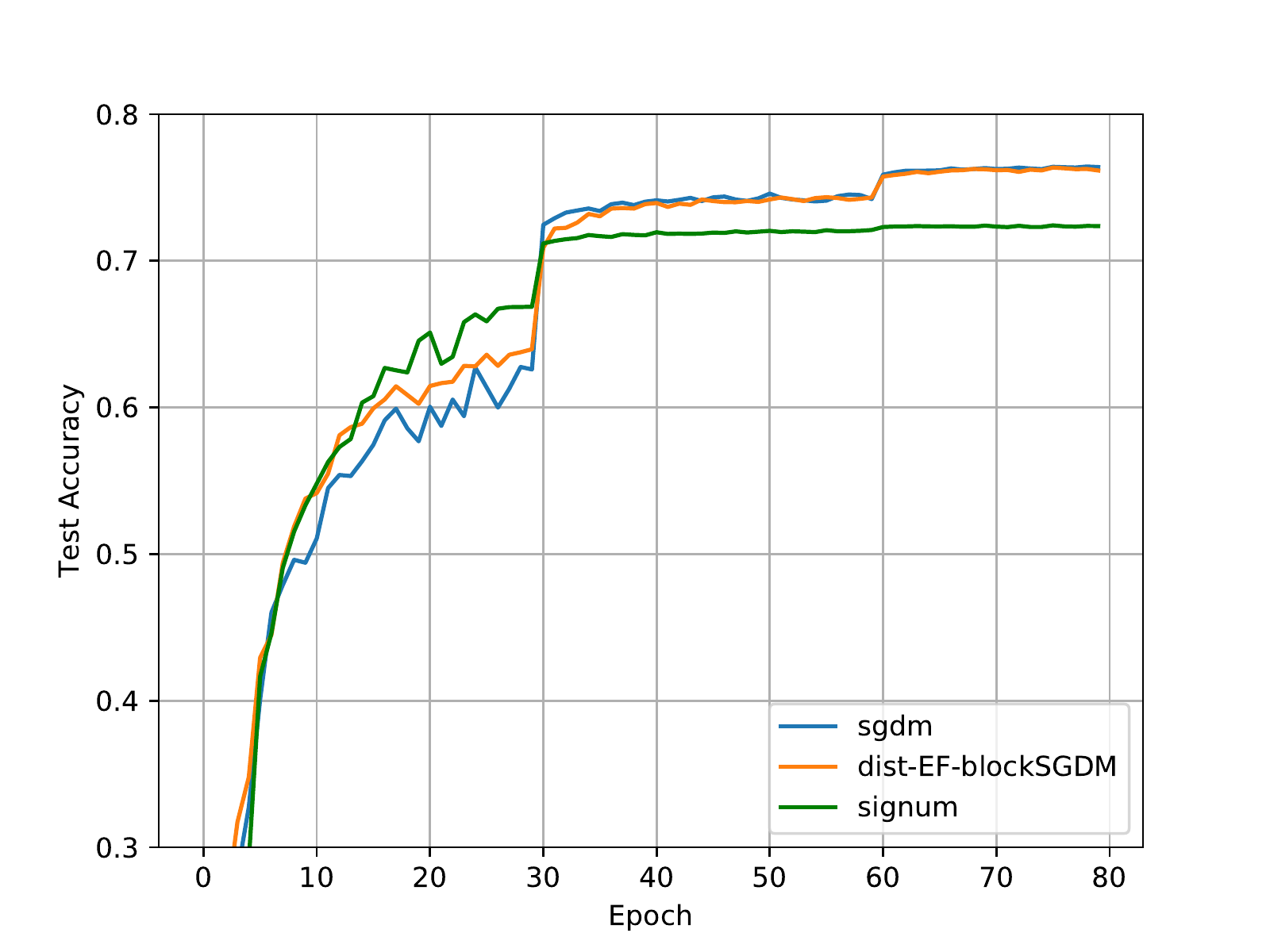}} 
\subfigure[Test accuracy w.r.t. time. \label{fig:t2}]{\includegraphics[width=0.32\columnwidth]{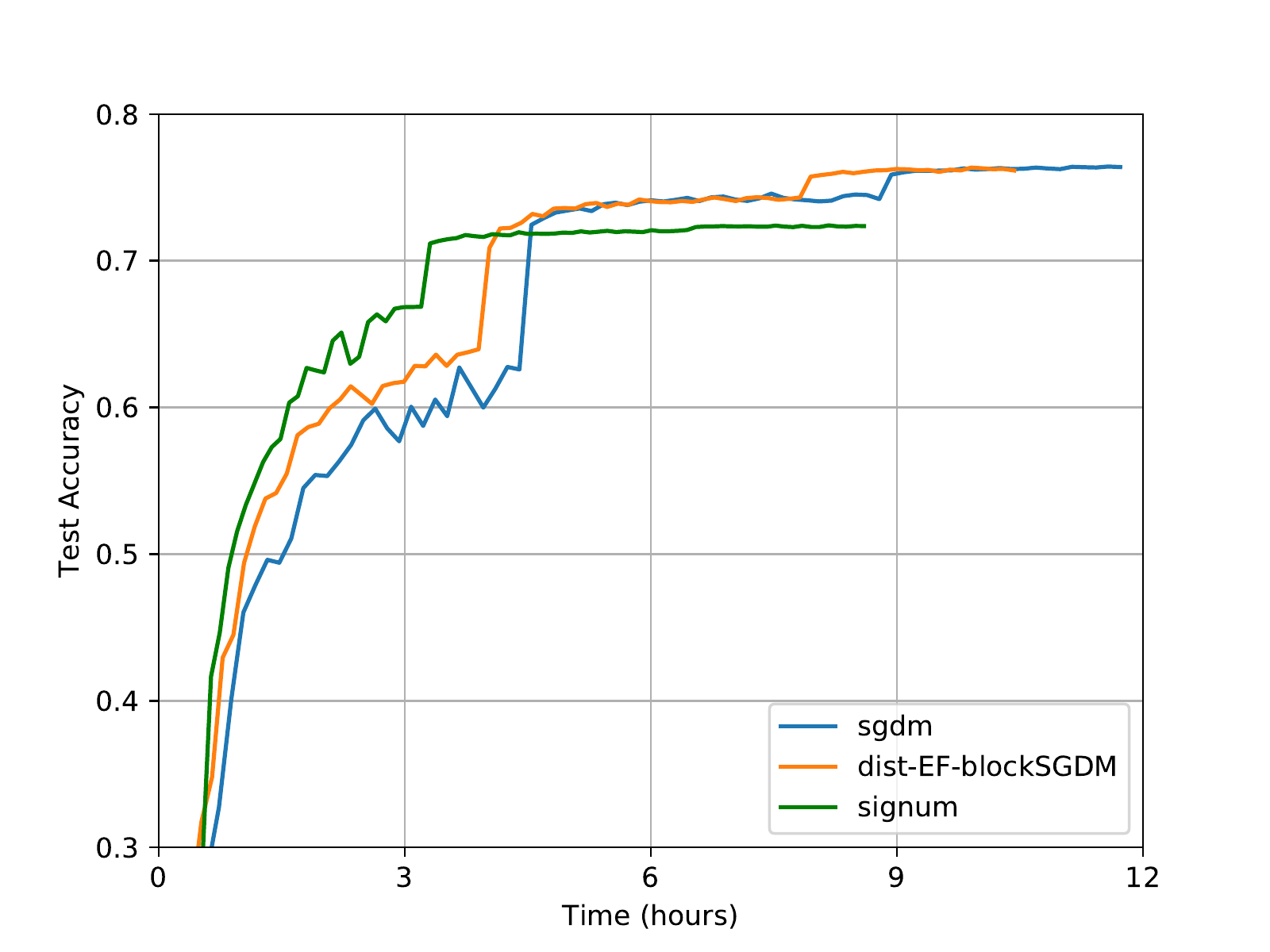}} 
\subfigure[Workload breakdown. \label{fig:load2}]{\includegraphics[width=0.32\columnwidth]{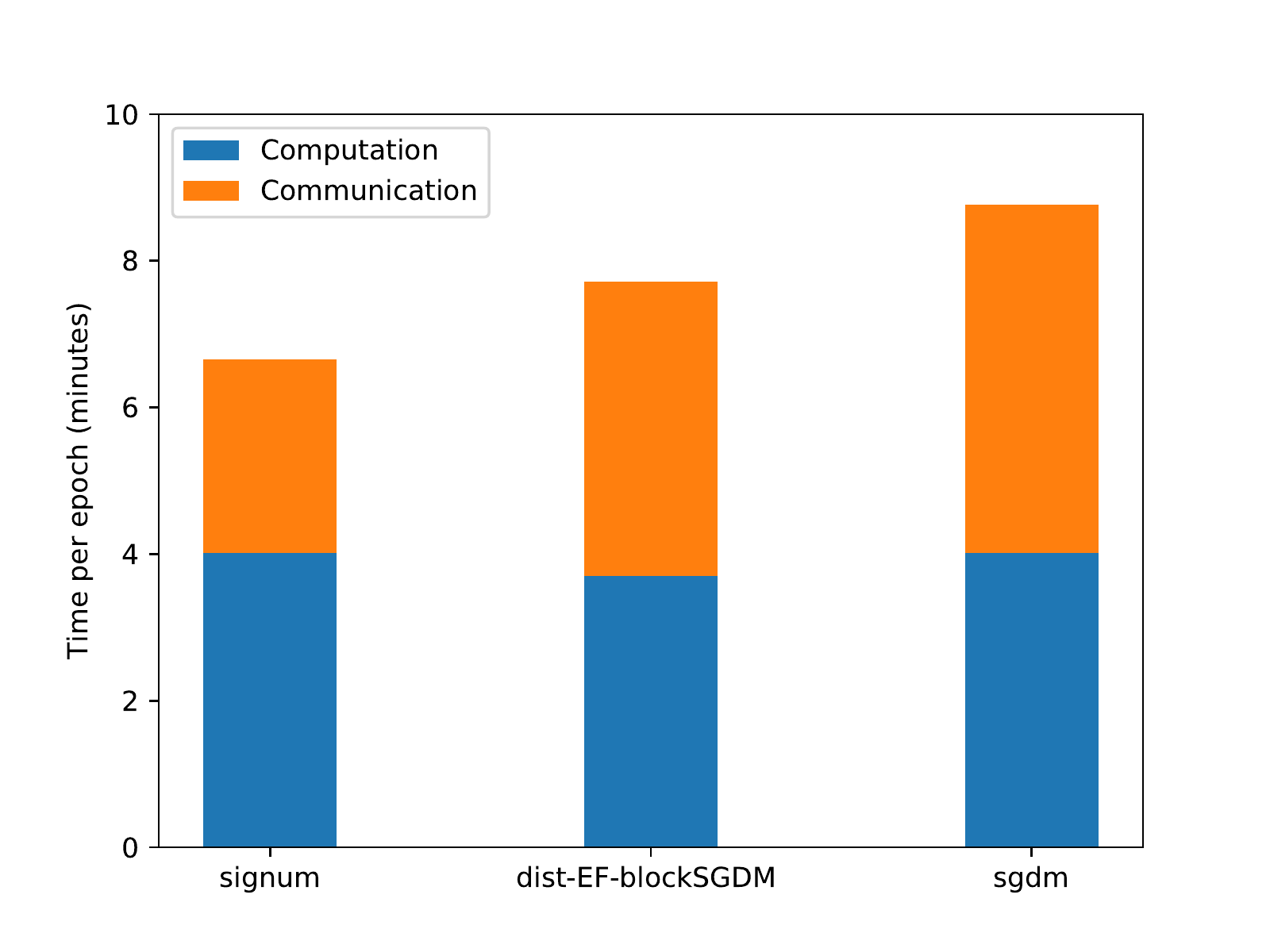}}
\vskip -.1in
\caption{Distributed training results on the ImageNet dataset. 
Top:
7 workers;
Bottom:
15 workers.}
\label{fig:dist_imagenet}
\end{center}
\vskip -.3in
\end{figure}

Figure~\ref{fig:dist_imagenet} shows 
the testing accuracy w.r.t. the number of epochs and wall clock time.
As in Section~\ref{sec:distef-cifar100-expt},
SGDM and dist-EF-blockSGDM
have comparable accuracies,
while signum is inferior. 
When 7 workers are used, dist-EF-blockSGDM 
has higher accuracy than SGDM 
(76.77\% vs
76.27\%). 
dist-EF-blockSGDM 
reaches SGDM's highest accuracy 
in around 13 hours, 
while SGDM takes 24 hours (Figure~\ref{fig:t1}), 
leading to a $46\%$ speedup. 
With 15 machines, the improvement is smaller (Figure~\ref{fig:t2}). This is because the burden on the parameter server is heavier. We expect 
comparable speedup with the 7-worker setting can be obtained by using more
parameter servers. In both cases, signum converges fast but the test
accuracies are about $4\%$ worse.

Figures~\ref{fig:load1}
and \ref{fig:load2}
show a breakdown 
of wall clock time 
into computation and communication time.\footnote{Following
\cite{bernstein2019signsgd}, communication time includes the extra computation time for
error feedback and compression.}
All methods have comparable computation costs,
but
signum and dist-EF-blockSGDM have lower communication costs than SGDM.
The communication costs 
for signum and dist-EF-blockSGDM are comparable 
for 7 workers, 
but for 15 workers
signum is lower.
We speculate that 
it is because 
the sign vectors and scaling factors are sent separately to the server in our
implementation, which causes more latency on the server with more workers. 
This may be alleviated if 
the two operations
are fused.

 \section{Conclusion}
 
In this paper, we 
proposed  a
distributed blockwise SGD algorithm with error feedback 
and momentum.
 By partitioning the gradients into blocks, we can transmit each block of gradient using 1-bit quantization with its average $\ell_1$-norm. 
 The proposed methods are communication-efficient and have the same convergence rates as 
 full-precision distributed SGD/SGDM for nonconvex objectives. 
Experimental results show that the proposed methods have
 fast convergence and achieve the same test accuracy as SGD/SGDM, while
 signSGD and signum only achieve much worse accuracies. 
 
\bibliography{distef}
\bibliographystyle{plain}


\newpage

\appendix

\section{Experimental Setup}

As we focus on synchronous distributed training, 
it is not necessary to compress
weight decay. 
In the experiment, for dist-EF-blockSGD, the weight decay is not added to $g_{t, i}$. Instead, we add it to $\tilde{\Delta}_t$. 
For dist-EF-blockSGDM, as momentum is additive, we maintain an extra momentum
$\tilde{m}_t$ for weight decay on each machine. Specifically, we perform the following update on each worker:
\begin{eqnarray*}
\tilde{m}_t &=& \mu\tilde{m}_{t-1} + \lambda x_{t},\\
x_{t+1} &=& x_t - \eta_t(\tilde{\Delta}_t + \mu\tilde{m}_t + \lambda x_{t}),
\end{eqnarray*}
where $\lambda$ is the weight decay parameter. 
In the experiment, the sign
is mapped to $\{-1, 1\}$ and takes 1 bit.  Note that the gradient sign has zero probability of being zero. 

\subsection{Setup: Multi-GPU Experiment on CIFAR-100} 
\label{app:cifar}

Each algorithm is run for 200 epochs. 
We only tune the initial stepsize, 
using
a validation set with 5K images that is carved out from the training set. 
For dist-EF-blockSGD (resp. dist-EF-blockSGDM), we use the stepsize tuned for SGD (resp. SGDM).  
The stepsize with the best validation set performance is used to run the algorithm on the full training set.  
The stepsize 
is divided by $10$ at the $100$-th and $150$-th epochs. 
The weight decay parameter is fixed to $0.0005$, and the momentum parameter $\mu$ is $0.9$. 
When mini-batch size is $16$ per worker, 
for both SGD and SGDM, the stepsize is tuned from $\{0.05, 0.1, 0.5, 1\}$, 
and for signSGD and signum, the stepsize is chosen from $\{0.0005, 0.001, 0.005, 0.01\}$. 
When we obtain the best stepsize  
$\eta_0$
tuned with mini-batch size $B=16$ per worker, 
for $B=8$,
the best stepsize
is selected from $\{\eta_0/2,\eta_0\}$; whereas
for $B=32$,
it is selected from $\{\eta_0,2\eta_0\}$.
The best stepsizes obtained are shown in Table~\ref{tbl:best_lr}


 \begin{table}[h]
 \begin{center}
 \caption{Best stepsizes obtained by grid search on a hold-out validation set. We reuse the obtained stepsizes tuned for SGD/SGDM for 
 dist-EF-blockSGD/dist-EF-blockSGDM. }
 \vskip .1in
 \begin{tabular}{c|c|c|c}
  \hline
 & \multicolumn{3}{|c}{mini-batch size per worker} \\
 \hline
 algorithm & $8$ & $16$ & $32$ \\
  \hline
 full-precision SGD & $0.25$ & $0.5$ & $1$ \\
 full-precision SGDM & $0.05$ & $0.05$ & $0.1$ \\
  dist-EF-blockSGD & $0.25$ & $0.5$ & $1$ \\
 dist-EF-blockSGDM & $0.05$ & $0.05$ & $0.1$ \\
  signSGD & $0.001$ & $0.001$ & $0.002$ \\
  signum & $0.0005$ & $0.0005$ & $0.0005$ \\
  \hline
 \end{tabular}
  \label{tbl:best_lr}
 \end{center}
 \end{table}

\subsection{Setup: Distributed Training on ImageNet}
\label{app:imagenet}

We use the default hyperparameters 
for SGDM and signum 
in the code base, which have been tuned for the
ImageNet experiment in \cite{bernstein2019signsgd}.
Specifically, the momentum parameter $\mu$ is $0.9$, and weight decay parameter is $0.0001$.
A mini-batch size of 128 per worker is employed. 

For SGDM, we use
$\eta = 0.1 M$ (used for SGDM on the ImageNet experiment in the code base). For signum, $\eta = 0.0001$ (used for signum on the ImageNet experiment in the code base) on 7 workers and $\eta = 0.0002$ on 15 workers. 
For dist-EF-blockSGDM, we also use $\mu = 0.9$ and a weight decay of $0.0001$.
Its stepsize $\eta$ is $0.1$ 
for 7 workers,\footnote{We observe that $\eta=0.1M$ is too large for dist-EF-blockSGDM, while SGDM with $\eta=0.1$ performs worse than SGDM with $\eta=0.1M$. } 
and 
$0.2$ 
for 15 workers.

\section{Proof of Lemmas~\ref{lemma:recurrence_ec_iterate} and \ref{lemma:recurrence_mec_iterate}}
\begin{lemma} 
Suppose that $p_{t, i} = z_{t, i} + \frac{\eta_{t-1}}{\eta_t}e_{t, i}$ for any sequence $z_{t, i}$. Consider the error-corrected iterate $\tx_t = x_t - \eta_{t-1}\left(\te_t + \frac{1}{M} \sum_{i=1}^M e_{t,i}\right)$, it satisfies the recurrence:
\begin{eqnarray*}
\tx_{t+1} = \tx_t - \eta_t\frac{1}{M} \sum_{i=1}^M z_{t,i}.
\end{eqnarray*}
\end{lemma}
\begin{proof}
\begin{align*}
\tilde{x}_{t+1} &= x_t - \eta_t\mathcal{C}(\tp_t) - \eta_t\te_{t+1} - \eta_t\frac{1}{M} \sum_{i=1}^M e_{t+1,i} \hspace{.54in} \text{Apply } x_{t+1} = x_t - \eta_t\mathcal{C}(\tp_t) \\
&= x_t - \eta_t\tp_t - \eta_t\frac{1}{M} \sum_{i=1}^M e_{t+1,i}  \hspace{1.27in} \text{Apply } \te_{t+1} = \tp_t -  \mathcal{C}(\tp_t) \\
&= x_t - \eta_t\frac{1}{M} \sum_{i=1}^M (\Delta_{t,i} + e_{t+1,i}) - \eta_{t-1}\tilde{e}_t \hspace{.64in} \text{Apply } \tp_t =  \frac{1}{M} \sum_{i=1}^M\Delta_{t,i} + \frac{\eta_{t-1}}{\eta_t}\te_t \\
&= x_t - \eta_t\frac{1}{M} \sum_{i=1}^M p_{t,i} - \eta_{t-1}\tilde{e}_t \hspace{1.28in} \text{Apply } e_{t+1, i} = p_{t, i} -  \Delta_{t, i} \\
&= x_t - \eta_t\frac{1}{M} \sum_{i=1}^M z_{t,i} - \eta_{t-1}\frac{1}{M} \sum_{i=1}^M e_{t,i} - \eta_{t-1}\tilde{e}_t \hspace{.25in} \text{Apply } p_{t, i} = z_{t, i} +  \frac{\eta_{t-1}}{\eta_t}e_{t, i} \\
&= \tilde{x}_t - \eta_t\frac{1}{M} \sum_{i=1}^M z_{t,i}.
\end{align*}
The Lemmas~\ref{lemma:recurrence_ec_iterate} and \ref{lemma:recurrence_mec_iterate} hold by substituting $z_{t,i} = g_{t, i}$ and $z_{t,i} = \mu m_{t, i} + g_{t, i}$, respectively.
\end{proof}


\section{Proof of Theorem~\ref{theorem:general_distefsgd_convergence}}


\begin{proof}
By the smoothness of the function $F$, we have
\begin{eqnarray*}
\lefteqn{\CE_t[F(\tx_{t+1})]} \\
& \leq & F(\tx_t) + \langle \nabla F(\tx_t), \CE_t[\tx_{t+1} - \tx_t] \rangle + \frac{L}{2}\CE_t\left[\|\tx_{t+1} - \tx_t\|^2_2\right] \\
& = & F(\tx_t) - \eta_t\left\langle \nabla F(\tx_t), \CE_t\left[\frac{1}{M} \sum_{i=1}^M g_{t,i}\right] \right\rangle + \frac{L\eta_t^2}{2}\CE_t\left[\left\|\frac{1}{M} \sum_{i=1}^M g_{t,i}\right\|^2_2\right] \\
& = & F(\tx_t) - \eta_t\left\langle \nabla F(\tx_t),  \nabla F(x_t) \right\rangle + \frac{L\eta_t^2}{2}\left\|\nabla F(x_t)\right\|^2_2 + \frac{L\eta_t^2}{2}\CE_t\left[\left\|\frac{1}{M} \sum_{i=1}^M g_{t,i} - \nabla F(x_t)\right\|^2_2\right] \\
& \leq & F(\tx_t) - \eta_t\left\langle \nabla F(\tx_t),  \nabla F(x_t) \right\rangle + \frac{L\eta_t^2}{2}\left\|\nabla F(x_t)\right\|^2_2 + \frac{L\eta_t^2\sigma^2}{2M}\\
\end{eqnarray*}
where in the second equality we use Lemma~\ref{lemma:recurrence_ec_iterate}, and the second-to-last inequality follows the fact $\CE[\|x - \CE[x]\|_2^2] = \CE[\|x\|_2^2] - \|\CE[x]\|_2^2$. In the last inequality, we use the variance bound of the mini-batch gradient, i.e., $\CE_t\left[\left\|\frac{1}{M} \sum_{i=1}^M g_{t,i} - \nabla F(x_t)\right\|^2_2\right] \leq \frac{\sigma^2}{M}$. Then, we get
\begin{eqnarray*}
\lefteqn{\CE_t[F(\tx_{t+1})]} \\
& \leq & F(\tx_t) - \eta_t\left\langle \nabla F(x_t),  \nabla F(x_t) \right\rangle + \frac{L\eta_t^2}{2}\left\|\nabla F(x_t)\right\|^2_2 + \frac{L\eta_t^2\sigma^2}{2M} 
+ \eta_t\left\langle \nabla F(x_t) - \nabla F(\tx_t),  \nabla F(x_t) \right\rangle \\
& = & F(\tx_t) - \eta_t\left( 1- \frac{L\eta_t}{2}\right)\left\|\nabla F(x_t)\right\|^2_2 + \frac{L\eta_t^2\sigma^2}{2M} 
+ \eta_t\left\langle \nabla F(x_t) - \nabla F(\tx_t),  \nabla F(x_t) \right\rangle\\
& \leq & F(\tx_t) - \eta_t\left( 1- \frac{L\eta_t}{2}\right)\left\|\nabla F(x_t)\right\|^2_2 + \frac{L\eta_t^2\sigma^2}{2M} 
+ \frac{\eta_t\rho}{2}\|\nabla F(x_t)\|_2^2 + \frac{\eta_t}{2\rho}\|\nabla F(x_t) - \nabla F(\tx_t)\|_2^2\\
& =& F(\tx_t) - \eta_t\left(1- \frac{L\eta_t + \rho}{2} \right)\left\|\nabla F(x_t)\right\|^2_2 + \frac{L\eta_t^2\sigma^2}{2M} 
+ \frac{\eta_t}{2\rho}\|\nabla F(x_t) - \nabla F(\tx_t)\|_2^2 \\
& \leq & F(\tx_t) - \eta_t\left(1- \frac{L\eta_t + \rho}{2} \right)\left\|\nabla F(x_t)\right\|^2_2 + \frac{L\eta_t^2\sigma^2}{2M} 
+ \frac{\eta_tL^2}{2\rho}\|x_t - \tx_t\|_2^2 \\
& = & F(\tx_t) - \eta_t\left(1- \frac{L\eta_t + \rho}{2} \right)\left\|\nabla F(x_t)\right\|^2_2 + \frac{L\eta_t^2\sigma^2}{2M} 
+ \frac{\eta_t\eta_{t-1}^2L^2}{2\rho}\left\|\te_t + \frac{1}{M} \sum_{i=1}^M e_{t,i}\right\|_2^2,
\end{eqnarray*}
where the second inequality follows from Young's inequality with $\rho > 0$. The last inequality follows from the smoothness of the function $F$. 
Let $\rho = 1/2$. Taking total expectation and using Lemma~\ref{lemma:error_bound} with $\mu=0$, we get
\begin{eqnarray*}
\lefteqn{\CE_t[F(\tx_{t+1})]} \\
& \leq & \CE[F(\tx_t)] - \eta_t\left(\frac{3}{4}- \frac{L\eta_t}{2} \right)\CE[\left\|\nabla F(x_t)\right\|^2_2] + \frac{L\eta_t^2\sigma^2}{2M} 
+ \frac{8L^2\eta_t\eta_{t-1}^2(1 - \delta)G^2}{\delta^2}\left[1 + \frac{16}{\delta^2}\right].
\end{eqnarray*}
Assume that $\eta_t < 3/(2L)$ for all $t$. Rearranging the terms, taking summation, and dividing by $\sum_{k=0}^{T-1}\frac{\eta_k}{4}\left(3 - 2L\eta_k \right)$ gives
\begin{eqnarray*}
\lefteqn{\frac{1}{\sum_{k=0}^{T-1}\eta_k\left(3 - 2L\eta_k \right)}\sum_{t=0}^{T-1}\eta_t\left(3 - 2L\eta_t \right)\CE\left[\left\|\nabla F(x_t)\right\|^2_2\right]} \\
& \leq & \frac{4}{\sum_{k=0}^{T-1}\eta_k\left(3 - 2L\eta_k \right)}\sum_{t=0}^{T-1}\CE[F(\tx_t) - F(\tx_{t+1})]  + \frac{2L\sigma^2}{M}\sum_{t=0}^{T-1}\frac{\eta_t^2}{\sum_{k=0}^{T-1}\eta_k\left(3 - 2L\eta_k \right)} \\
&& + \frac{32L^2(1 - \delta)G^2}{\delta^2}\left[1 + \frac{16}{\delta^2}\right]\sum_{t=0}^{T-1}\frac{\eta_t\eta_{t-1}^2}{\sum_{k=0}^{T-1}\eta_k\left(3 - 2L\eta_k \right)}\\
& \leq & \frac{4}{\sum_{k=0}^{T-1}\eta_k\left(3 - 2L\eta_k \right)}[F(x_0) - F_*]  + \frac{2L\sigma^2}{M}\sum_{t=0}^{T-1}\frac{\eta_t^2}{\sum_{k=0}^{T-1}\eta_k\left(3 - 2L\eta_k \right)} \\
&& + \frac{32L^2(1 - \delta)G^2}{\delta^2}\left[1 + \frac{16}{\delta^2}\right]\sum_{t=0}^{T-1}\frac{\eta_t\eta_{t-1}^2}{\sum_{k=0}^{T-1}\eta_k\left(3 - 2L\eta_k \right)}.
\end{eqnarray*}
Let $o \in \{0, \dots, T-1\}$ be an index such that 
\begin{eqnarray*}
P(o = k) = \frac{\eta_k\left(3 - 2L\eta_k \right)}{\sum_{t=0}^{T-1}\eta_t\left(3 - 2L\eta_t \right)}.
\end{eqnarray*}
Then, we have 
\begin{eqnarray*}
\CE[\|\nabla F(x_o)\|_2^2]= \frac{1}{\sum_{k=0}^{T-1}\eta_k\left(3 - 2L\eta_k \right)}\sum_{t=0}^{T-1}\eta_t\left(3 - 2L\eta_t \right)\CE\left[\left\|\nabla F(x_t)\right\|^2_2\right],
\end{eqnarray*}
which concludes the results.
\end{proof}


%

\section{Proof of Corollary~\ref{corollary:convergence-dist-ef-sgd}}
 \begin{proof}
Let $\eta_t = \eta$ for all $t$, we have
\begin{eqnarray}
\CE[\|\nabla F(x_o)\|_2^2]
& \leq & \frac{4}{\eta\left(3- 2L\eta \right)T}[F(x_0) - F_*]  + \frac{2L\eta\sigma^2}{\left(3- 2L\eta\right)M} \nonumber \\
&& + \frac{32L^2\eta^2(1 - \delta)G^2}{\left(3- 2L\eta \right)\delta^2}\left[1 + \frac{16}{\delta^2}\right]. \label{eq:general_bound_distefsgd}
\end{eqnarray}
Let $\eta = \min\left(\frac{1}{2L}, \frac{\gamma}{\frac{\sqrt{T}}{\sqrt{M}} + \frac{(1 - \delta)^{1/3}}{\delta^{2/3}}\left(1 + \frac{16}{\delta^2}\right)^{1/3}T^{1/3}}\right)$ for some $\gamma > 0$, then $3- 2L\eta \geq 2$. Substituting this into (\ref{eq:general_bound_distefsgd}), we get
\begin{eqnarray*}
\lefteqn{\CE[\|\nabla F(x_o)\|_2^2]} \\
& \leq & \frac{2}{\eta T}[F(x_0) - F_*]  + \frac{L\eta\sigma^2}{M} + \frac{16L^2\eta^2(1 - \delta)G^2}{\delta^2}\left[1 + \frac{16}{\delta^2}\right] \\
& \leq & \frac{2}{T}\max\left(2L, \frac{\sqrt{T}}{\gamma\sqrt{M}} + \frac{(1 - \delta)^{1/3}}{\gamma\delta^{2/3}}\left[1 + \frac{16}{\delta^2}\right]^{1/3}T^{1/3}\right)[F(x_0) - F_*]  \\
&& + \frac{L\eta\sigma^2}{M} + \frac{16L^2\eta^2(1 - \delta)G^2}{\delta^2}\left[1 + \frac{16}{\delta^2}\right] \\
& \leq & \frac{4L}{T}[F(x_0) - F_*] + \left[\frac{2}{\gamma\sqrt{MT}} + \frac{2(1 - \delta)^{1/3}}{\gamma\delta^{2/3}T^{2/3}}\left[1 + \frac{16}{\delta^2}\right]^{1/3} \right][F(x_0) - F_*] \\
&& + \frac{L\gamma\sigma^2}{\sqrt{MT}} + \frac{16L^2\gamma^2(1 - \delta)^{1/3}G^2}{\delta^{2/3}T^{2/3}}\left[1 + \frac{16}{\delta^2}\right]^{1/3} \\
& = & \frac{4L}{T}[F(x_0) - F_*] + \left[\frac{2}{\gamma}[F(x_0) - F_*] +  L\gamma\sigma^2\right]\frac{1}{\sqrt{MT}}  \\
&& +  \frac{2(1 - \delta)^{1/3}[\frac{1}{\gamma}[F(x_0) - F_*] + 8L^2\gamma^2G^2]}{\delta^{2/3}T^{2/3}}\left[1 + \frac{16}{\delta^2}\right]^{1/3}.
\end{eqnarray*}
The bound on full-precision distributed SGD follows similar proof. For completeness, we present proof here. 
By the smoothness of the function $F$, we have
\begin{eqnarray*}
\lefteqn{\CE_t[F(x_{t+1})]} \\
& \leq & F(x_t) + \langle \nabla F(x_t), \CE_t[x_{t+1} - x_t] \rangle + \frac{L}{2}\CE_t\left[\|x_{t+1} - x_t\|^2_2\right] \\
& = & F(x_t) - \eta_t\left\langle \nabla F(x_t), \CE_t\left[\frac{1}{M} \sum_{i=1}^M g_{t,i}\right] \right\rangle + \frac{L\eta_t^2}{2}\CE_t\left[\left\|\frac{1}{M} \sum_{i=1}^M g_{t,i}\right\|^2_2\right] \\
& = & F(\tx_t) - \eta_t\left(1- \frac{L\eta_t}{2}\right)\left\|\nabla F(x_t)\right\|^2_2 + \frac{L\eta_t^2}{2}\CE_t\left[\left\|\frac{1}{M} \sum_{i=1}^M g_{t,i} - \nabla F(x_t)\right\|^2_2\right] \\
& \leq & F(x_t) - \eta_t\left(1- \frac{L\eta_t}{2}\right)\left\|\nabla F(x_t)\right\|^2_2  + \frac{L\eta_t^2\sigma^2}{2M}.
\end{eqnarray*}
Let $\eta_t = \eta$. Taking total expectation, rearranging terms, and averaging over $T$, we obtain
\begin{eqnarray*}
\CE\left[\left\|\nabla F(x_o)\right\|^2_2\right]
& \leq & \frac{2}{\eta\left(2 - L\eta \right)T}[F(x_0) - F_*]  + \frac{L\eta\sigma^2}{(2 - L\eta)M}.
\end{eqnarray*}
Substituting $\eta = \min\left(\frac{1}{2L}, \frac{\gamma\sqrt{M}}{\sqrt{T}}\right)$, we get
\begin{eqnarray*}
\CE\left[\left\|\nabla F(x_o)\right\|^2_2\right]
& \leq & \frac{4}{3\eta T}[F(x_0) - F_*]  + \frac{2L\eta\sigma^2}{3M} \\
& \leq & \frac{8L}{3T}[F(x_0) - F_*] + \frac{4}{3\gamma\sqrt{T}}[F(x_0) - F_*]  + \frac{2L\gamma\sigma^2}{3M\sqrt{T}} \\
& = & \frac{8L}{3T}[F(x_0) - F_*] + \left[\frac{2}{\gamma}[F(x_0) - F_*]  + L\gamma\sigma^2\right]\frac{2}{3\sqrt{MT}}.
\end{eqnarray*}

\end{proof}

\section{Proof of Corollary~\ref{corollary:convergence-decrease-dist-ef-sgd}}
 \begin{proof}
 Let $\eta_t = \frac{\gamma}{\frac{((t+1)T)^{1/4}}{\sqrt{M}} + \frac{(1 - \delta)^{1/3}}{\delta^{2/3}}\left(1 + \frac{16}{\delta^2}\right)^{1/3}T^{1/3}}$. The following implies that $\eta_t \leq 1/(2L)$ for all $0 \leq t \leq T-1$.
  \begin{eqnarray*} \label{eq:dc_stepsize_condition}
T \geq 16L^4\gamma^4M^2.
\end{eqnarray*} 
Then, we have
\begin{eqnarray*}
\sum_{t=0}^{T-1} \eta_t & = & \gamma\sum_{t=0}^{T-1} \frac{1}{\frac{((t+1)T)^{1/4}}{\sqrt{M}} + \frac{(1 - \delta)^{1/3}}{\delta^{2/3}}\left(1 + \frac{16}{\delta^2}\right)^{1/3}T^{1/3}} \\
& \geq  & \gamma\sum_{t=0}^{T-1} \frac{1}{\frac{\sqrt{T}}{\sqrt{M}} + \frac{(1 - \delta)^{1/3}}{\delta^{2/3}}\left(1 + \frac{16}{\delta^2}\right)^{1/3}T^{1/3}} \\
& = &  \frac{1}{\frac{1}{\gamma\sqrt{MT}} + \frac{(1 - \delta)^{1/3}}{\gamma\delta^{2/3}T^{2/3}}\left(1 + \frac{16}{\delta^2}\right)^{1/3}}.
\end{eqnarray*}
 Using the fact that $\sum_{t=1}^Tt^{\alpha - 1} \leq \int_{0}^Tx^{\alpha - 1}dx = \frac{T^{\alpha}}{\alpha}$, for any $0< \alpha < 1$, we have
  \begin{eqnarray*}
  \sum_{t=0}^{T-1} \eta_t^2 & \leq & \frac{\gamma^2M}{\sqrt{T}}\sum_{t=1}^{T}\frac{1}{\sqrt{t}} \leq 2\gamma^2M, \\
  \sum_{t=0}^{T-1} \eta_t\eta_{t-1}^2 & = &  \sum_{t=1}^{T-1} \eta_t\eta_{t-1}^2 \leq   \sum_{t=1}^{T-1}\eta_{t-1}^3  \leq  \frac{\gamma^3}{\frac{(1 - \delta)}{\delta^{2}}\left(1 + \frac{16}{\delta^2}\right)}.
   \end{eqnarray*}
   Substituting the above results into Theorem~\ref{theorem:general_distefsgd_convergence}, we obtain
   \begin{eqnarray*}
\lefteqn{\CE\left[\left\|\nabla F(x_o)\right\|^2_2\right]}\\
& \leq & \left[\frac{1}{\gamma\sqrt{MT}} + \frac{(1 - \delta)^{1/3}}{\gamma\delta^{2/3}T^{2/3}}\left(1 + \frac{16}{\delta^2}\right)^{1/3}\right]2[F(x_0) - F_*]  \\
&& + \left[\frac{1}{\gamma\sqrt{MT}} + \frac{(1 - \delta)^{1/3}}{\gamma\delta^{2/3}T^{2/3}}\left(1 + \frac{16}{\delta^2}\right)^{1/3}\right]2L\gamma^2\sigma^2 \\
&& + \left[\frac{1}{\gamma\sqrt{MT}} + \frac{(1 - \delta)^{1/3}}{\gamma\delta^{2/3}T^{2/3}}\left(1 + \frac{16}{\delta^2}\right)^{1/3}\right]16L^2\gamma^3G^2 \\
& = & 2\left[\frac{1}{\sqrt{MT}} + \frac{(1 - \delta)^{1/3}}{\delta^{2/3}T^{2/3}}\left(1 + \frac{16}{\delta^2}\right)^{1/3}\right]\left[\frac{1}{\gamma}[F(x_0) - F_*] + L\gamma\sigma^2 + 8L^2\gamma^2G^2\right].
\end{eqnarray*}
Similarly, let $\eta_t = \frac{\gamma\sqrt{t+1}}{\frac{T}{\sqrt{M}} + \frac{(1 - \delta)^{1/3}}{\delta^{2/3}}\left(1 + \frac{16}{\delta^2}\right)^{1/3}T^{5/6}}$. 
We obtain 
\begin{eqnarray*}
\sum_{t=0}^{T-1} \eta_t & = & \gamma\sum_{t=0}^{T-1} \frac{\sqrt{t+1}}{\frac{T}{\sqrt{M}} + \frac{(1 - \delta)^{1/3}}{\delta^{2/3}}\left(1 + \frac{16}{\delta^2}\right)^{1/3}T^{5/6}} \\
& = & \gamma\sum_{t=1}^{T} \frac{\sqrt{t}}{\frac{T}{\sqrt{M}} + \frac{(1 - \delta)^{1/3}}{\delta^{2/3}}\left(1 + \frac{16}{\delta^2}\right)^{1/3}T^{5/6}} \\
& \geq & \gamma\int_{0}^{T} \frac{\sqrt{x}}{\frac{T}{\sqrt{M}} + \frac{(1 - \delta)^{1/3}}{\delta^{2/3}}\left(1 + \frac{16}{\delta^2}\right)^{1/3}T^{5/6}}dx \\
& = & \frac{2T^{3/2}}{\frac{3T}{\gamma\sqrt{M}} + \frac{3(1 - \delta)^{1/3}}{\gamma\delta^{2/3}}\left(1 + \frac{16}{\delta^2}\right)^{1/3}T^{5/6}}.
\end{eqnarray*}
Using the fact that $\sum_{t=1}^Tt^{\alpha} \leq \int_{1}^{T+1}x^{\alpha}dx \leq \frac{(T+1)^{\alpha + 1}}{\alpha + 1}$ for any $\alpha > 0$, we also have
  \begin{eqnarray*}
  \sum_{t=0}^{T-1} \eta_t^2 & \leq & \frac{\gamma^2M}{T^2}\sum_{t=1}^{T}t = \frac{\gamma^2M(T+1)}{2T}, \\
  \sum_{t=0}^{T-1} \eta_t\eta_{t-1}^2 & = &  \sum_{t=1}^{T-1} \eta_t\eta_{t-1}^2 \leq   \sum_{t=1}^{T-1}\eta_{t}^3  \leq  \frac{2\gamma^3(T+1)^{5/2}}{5\frac{(1 - \delta)}{\delta^{2}}\left(1 + \frac{16}{\delta^2}\right)T^{5/2}}.
  \end{eqnarray*}
  Assuming that $T \geq 4L^2\gamma^2M$, we have $\eta_t \leq 1/(2L)$ for all $0 \leq t \leq T-1$. Substituting the above results into Theorem~\ref{theorem:general_distefsgd_convergence}, we obtain
     \begin{eqnarray*}
\lefteqn{\CE\left[\left\|\nabla F(x_o)\right\|^2_2\right]}\\
& \leq & \left[\frac{1}{\gamma\sqrt{MT}} + \frac{(1 - \delta)^{1/3}}{\gamma\delta^{2/3}T^{2/3}}\left(1 + \frac{16}{\delta^2}\right)^{1/3}\right]3[F(x_0) - F_*]  \\
&& + \left[\frac{1}{\gamma\sqrt{MT}} + \frac{(1 - \delta)^{1/3}}{\gamma\delta^{2/3}T^{2/3}}\left(1 + \frac{16}{\delta^2}\right)^{1/3}\right]\frac{3L\gamma^2\sigma^2(T+1)}{4T} \\
&& + \left[\frac{1}{\gamma\sqrt{MT}} + \frac{(1 - \delta)^{1/3}}{\gamma\delta^{2/3}T^{2/3}}\left(1 + \frac{16}{\delta^2}\right)^{1/3}\right]\frac{48L^2\gamma^3G^2(T+1)^{5/2}}{5T^{5/2}} \\
& = & 3\left[\frac{1}{\sqrt{MT}} + \frac{(1 - \delta)^{1/3}}{\delta^{2/3}T^{2/3}}\left(1 + \frac{16}{\delta^2}\right)^{1/3}\right]\left[\frac{1}{\gamma}[F(x_0) - F_*] + \frac{L\gamma\sigma^2(T+1)}{4T}  + \frac{16L^2\gamma^2G^2(T+1)^{5/2}}{5T^{5/2}}\right].
\end{eqnarray*}
  \end{proof}


\section{Proof of Proposition~\ref{prop:block_density}}
 \begin{proof}
 \begin{eqnarray*}
 \|\mC_B(v) - v\|_2^2 & = & \sum_{b=1}^B\left\|\frac{\|v_{\mG_b}\|_1}{d_b}\text{sign}(v_{\mG_b}) - v_{\mG_b}\right\|_2^2 \\
 & = & \sum_{b=1}^B\left[\frac{\|v_{\mG_b}\|_1^2}{d_b} - 2\frac{\|v_{\mG_b}\|_1^2}{d_b} + \|v_{\mG_b}\|_2^2\right] \\
 & = & \sum_{b=1}^B\left(1 - \frac{\|v_{\mG_b}\|_1^2}{d_b\|v_{\mG_b}\|_2^2}\right)\|v_{\mG_b}\|_2^2 \\
 & \leq & \left(1 - \min_{b \in [B]}\frac{\|v_{\mG_b}\|_1^2}{d_b\|v_{\mG_b}\|_2^2}\right)\|v\|_2^2.
 \end{eqnarray*}
 \end{proof}

\section{Proof of Theorem~\ref{theorem:general_distefsgdm_convergence}}

We first introduce the following Lemmas.

\begin{lemma} \label{lemma:m_bound}
For any $i \in [M]$, we have
\begin{eqnarray*}
\CE\left[\|\mu m_{t, i} + g_{t, i}\|_2^2\right] \leq \frac{G^2}{1 - \mu}.
\end{eqnarray*}
\end{lemma}
\begin{proof}
\begin{eqnarray*}
\CE\left[\|\mu m_{t, i} + g_{t, i}\|_2^2\right] & = & \CE\left[\left\|\sum_{k=1}^t\mu^{t-k+1} g_{k, i} + g_{t, i}\right\|_2^2\right] \\
& = & \left(\sum_{k=1}^t\mu^{t-k+1} + 1\right)^2\CE\left[\left\|\frac{\sum_{k=1}^t\mu^{t-k+1} g_{k, i} + g_{t, i}}{\sum_{k=1}^t\mu^{t-k+1} + 1}\right\|_2^2\right] \\
& \leq & \left(\sum_{k=1}^t\mu^{t-k+1} + 1\right)\left(\sum_{k=1}^t\mu^{t-k+1}\CE\left[\|g_{k, i}\|_2^2\right] + \CE\left[\|g_{t, i}\|_2^2\right]\right) \\
& \leq & \left(\sum_{k=1}^t\mu^{t-k+1} + 1\right)^2G^2 \\
& \leq & \frac{G^2}{(1 - \mu)^2},
\end{eqnarray*}
where in the first inequality we use Jensen's inequality. In the second-to-last equality, we apply Assumptions~\ref{assumption:ef_bounded_var} and \ref{assumption:ef_bounded_grad}. 
The last inequality follows from the sum of a geometric series. 
\end{proof}


\begin{lemma} \label{lemma:error_bound}
For any $t \geq 0$, we have
\begin{eqnarray*}
\CE\left[\left\|\te_{t} + \frac{1}{M} \sum_{i=1}^M e_{t,i}\right\|_2^2\right] \leq  \frac{8(1-\delta)G^2}{\delta^2(1-\mu)^2}\left[1 + \frac{16}{\delta^2}\right].
\end{eqnarray*}
\end{lemma}
\begin{proof}
When $t = 0$, the bound trivially holds as $\te_{0} = 0$ and $e_{0,i} = 0$ for all $i$. 
Using $(a + b)^2 \leq 2a^2 + 2b^2$, we get
\begin{eqnarray}
\left\|\te_{t+1} + \frac{1}{M} \sum_{i=1}^M e_{t+1,i}\right\|_2^2 & \leq & 2\left\|\te_{t+1}\right\|_2^2 + 2\left\|\frac{1}{M} \sum_{i=1}^M e_{t+1,i}\right\|_2^2 \nonumber\\
& \leq & 2\left\|\te_{t+1}\right\|_2^2 + \frac{2}{m} \sum_{i=1}^M\left\| e_{t+1,i}\right\|_2^2, \hspace{.1in} \forall t \geq 0. \label{eq:aet_0}
\end{eqnarray}
Now, we can consider two terms separately. For the second term, we have
\begin{eqnarray}
\lefteqn{\frac{1}{M}\sum_{i=1}^M\CE\left[\left\|e_{t+1,i}\right\|_2^2\right] = \frac{1}{M} \sum_{i=1}^M\CE\left[\left\|\mathcal{C}(p_{t,i}) - p_{t, i}\right\|_2^2\right]} \nonumber\\
& \leq &  (1 - \delta)\frac{1}{M} \sum_{i=1}^M\CE\left[\left\|p_{t, i}\right\|_2^2\right] \label{eq:aet_1}\\
& = & (1 - \delta)\frac{1}{M} \sum_{i=1}^M\CE\left[\left\|e_{t, i} + \mu m_{t, i} + g_{t, i}\right\|_2^2\right] \nonumber\\
& \leq & (1 - \delta)(1 + \beta)\frac{1}{M} \sum_{i=1}^M\CE\left[\left\|e_{t, i}\right\|_2^2\right]  + (1 - \delta)(1 + 1/\beta)\frac{1}{M} \sum_{i=1}^M\CE\left[\left\|\mu m_{t, i} + g_{t, i}\right\|_2^2\right] \nonumber\\
& \leq & (1 - \delta)(1 + \beta)\frac{1}{M} \sum_{i=1}^M\CE\left[\left\|e_{t, i}\right\|_2^2\right]  + (1 - \delta)(1 + 1/\beta)\frac{G^2}{(1 - \mu)^2} \nonumber\\
& \leq & \sum_{k=0}^t[(1 - \delta)(1 + \beta)]^{t - k}(1 - \delta)(1 + 1/\beta)\frac{G^2}{(1 - \mu)^2} \nonumber\\
& \leq & \frac{(1 - \delta)(1 + 1/\beta)}{1 - (1 - \delta)(1 + \beta)}\frac{G^2}{(1 - \mu)^2}  =  \frac{(1 - \delta)(1 + 1/\beta)}{\delta - \beta(1 - \delta)}\frac{G^2}{(1 - \mu)^2}, \nonumber
\end{eqnarray}
where the first inequality follows from the definition of the compressor $\mC$. The second inequality follows from Young's inequality with any $\beta > 0$, 
and the third inequality follows from Lemma~\ref{lemma:m_bound}.   
The third equality follows from the definition of $p_{t, i}$ and the assumption $\eta_t = \eta$. The last inequality follows from the sum of a geometric series.  
Let $\beta = \frac{\delta}{2(1 - \delta)}$, then $1 + 1/\beta = (2 - \delta)/\delta \leq 2/\delta$. We get
\begin{eqnarray}
\frac{1}{M}\sum_{i=1}^M\CE\left[\left\|e_{t+1,i}\right\|_2^2\right] \leq  \frac{(1 - \delta)(1 + 1/\beta)}{\delta - \beta(1 - \delta)}\frac{G^2}{(1 - \mu)^2} =  \frac{2(1 - \delta)(1 + 1/\beta)}{\delta(1-\mu)^2}G^2 \leq \frac{4(1-\delta)}{\delta^2(1-\mu)^2}G^2. \label{eq:aet_2}
\end{eqnarray}
Then, the first term can be bounded as
\begin{eqnarray}
\lefteqn{\CE\left[\|\tilde{e}_{t+1} \|_2^2\right] = \CE\left[\| \mathcal{C}(\tp_t) - \tp_t \|_2^2\right] \leq (1 - \delta) \CE\left[\| \tp_t \|_2^2\right] } \nonumber\\
& = & (1 - \delta) \CE\left[\left\| \frac{1}{M} \sum_{i=1}^M \Delta_{t,i} + \tilde{e}_t \right\|_2^2\right] \nonumber\\
&\leq &  (1 - \delta)(1 + \beta)\CE\left[\| \tilde{e}_t \|_2^2\right] + (1 - \delta)(1 + 1 / \beta)\CE\left[\left\| \frac{1}{M} \sum_{i=1}^M \Delta_{t,i} \right\|_2^2\right]  \nonumber\\
&\leq & (1 - \delta)(1 + \beta)\CE\left[\| \tilde{e}_t \|_2^2\right] + 2(1 - \delta)(1 + 1 / \beta)\CE\left[\left\| \frac{1}{M} \sum_{i=1}^M \Delta_{t,i}  - \frac{1}{M} \sum_{i=1}^M p_{t, i}\right\|_2^2\right]  \nonumber\\
&& + 2(1 - \delta)(1 + 1 / \beta)\CE\left[\left\|\frac{1}{M} \sum_{i=1}^M p_{t, i}\right\|_2^2\right] \nonumber\\
&\leq & (1 - \delta)(1 + \beta)\CE\left[\| \tilde{e}_t \|_2^2\right] + 2(1 - \delta)^2(1 + 1 / \beta)\frac{1}{M} \sum_{i=1}^M\CE\left[\left\|p_{t, i}\right\|_2^2\right]   \nonumber\\
&& + 2(1 - \delta)(1 + 1 / \beta)\frac{1}{M} \sum_{i=1}^M \CE\left[\left\|p_{t, i}\right\|_2^2\right] \nonumber\\
&= & (1 - \delta)(1 + \beta)\CE\left[\| \tilde{e}_t \|_2^2\right] + 2(1 - \delta)(2 - \delta)(1 + 1 / \beta)\frac{1}{M} \sum_{i=1}^M \CE\left[\left\|p_{t, i}\right\|_2^2\right]. \label{eq:aet_3}
\end{eqnarray}
Combining (\ref{eq:aet_1}), (\ref{eq:aet_2}), we have $\frac{1}{M} \sum_{i=1}^M\CE\left[\left\|p_{t, i}\right\|_2^2\right] \leq \frac{4}{\delta^2(1-\mu)^2}G^2$. Substituting it into (\ref{eq:aet_3}), we get
\begin{eqnarray}
\lefteqn{\CE\left[\|\tilde{e}_{t+1} \|_2^2\right] } \nonumber\\
& \leq & (1 - \delta)(1 + \beta)\CE\left[\| \tilde{e}_t \|_2^2\right] + \frac{8(1 - \delta)(2 - \delta)(1 + 1 / \beta)}{\delta^2(1-\mu)^2}G^2 \nonumber\\
& \leq & \sum_{k=0}^t[(1 - \delta)(1 + \beta)]^{t - k}\frac{8(1 - \delta)(2 - \delta)(1 + 1 / \beta)}{\delta^2(1-\mu)^2}G^2 \nonumber\\
& \leq & \frac{8(1 - \delta)(2 - \delta)(1 + 1 / \beta)}{\delta^2(1 - (1 - \delta)(1 + \beta))(1-\mu)^2}G^2 \nonumber\\
& = & \frac{8(1 - \delta)(2 - \delta)(1 + 1 / \beta)}{\delta^2(\delta - \beta(1 - \delta))(1-\mu)^2}G^2 \nonumber\\
& = & \frac{16(1 - \delta)(2 - \delta)(1 + 1 / \beta)}{\delta^3(1-\mu)^2}G^2 \nonumber\\
& = & \frac{32(1 - \delta)(2 - \delta)}{\delta^4(1-\mu)^2}G^2 \nonumber\\
& \leq & \frac{64(1 - \delta)}{\delta^4(1-\mu)^2}G^2. \label{eq:aet_4}
\end{eqnarray}
Then, combining (\ref{eq:aet_0}), (\ref{eq:aet_2}) and (\ref{eq:aet_4}), we obtain
\begin{eqnarray*}
\CE\left[\left\|\te_{t+1} + \frac{1}{M} \sum_{i=1}^M e_{t+1,i}\right\|_2^2\right] \leq  \frac{8(1-\delta)G^2}{\delta^2(1-\mu)^2}\left[1 + \frac{16}{\delta^2}\right].
\end{eqnarray*}
\end{proof}


\begin{proof}
In the sequel, we assume $\eta_t = \eta$ for some $\eta > 0$. Let us introduce the following virtual iterate:
\begin{eqnarray*}
z_t & = & \tx_t -  \frac{\eta\mu^2}{1 - \mu}\frac{1}{M} \sum_{i=1}^Mm_{t-1, i},
\end{eqnarray*}
where $\tx_t$ is defined in Lemma \ref{lemma:recurrence_mec_iterate} .
Then, it satisfies the following recurrence:
\begin{eqnarray*}
z_{t+1} & = & \tx_{t + 1} - \frac{\eta\mu^2}{1 - \mu}\frac{1}{M} \sum_{i=1}^Mm_{t, i}  \\
& = & \tx_t - \eta\frac{1}{M} \sum_{i=1}^M(\mu m_{t, i} + g_{t, i}) - \frac{\eta\mu^2}{1 - \mu}\frac{1}{M} \sum_{i=1}^Mm_{t, i} \\
& = & \tx_t - \frac{\eta\mu}{1 - \mu}\frac{1}{M} \sum_{i=1}^Mm_{t, i} - \eta\frac{1}{M} \sum_{i=1}^Mg_{t, i} \\
& = & \tx_t - \frac{\eta\mu^2}{1 - \mu}\frac{1}{M} \sum_{i=1}^Mm_{t-1, i}- \frac{\eta\mu}{1 - \mu}\frac{1}{M} \sum_{i=1}^Mg_{t, i} - \eta\frac{1}{M} \sum_{i=1}^Mg_{t, i} \\
& = & z_{t} -  \frac{\eta}{1 - \mu}\frac{1}{M} \sum_{i=1}^Mg_{t, i}.
\end{eqnarray*}
By the smoothness of the function $F$, we get
\begin{eqnarray}
\lefteqn{\CE_t[F(z_{t+1})]} \nonumber\\
& \leq & F(z_t) + \langle \nabla F(z_t), \CE_t[z_{t+1} - z_t] \rangle + \frac{L}{2}\CE_t[\|z_{t+1} - z_t\|^2_2] \nonumber\\
& = & F(z_t) - \frac{\eta}{1 - \mu}\left\langle \nabla F(z_t), \CE_t\left[\frac{1}{M} \sum_{i=1}^Mg_{t, i}\right] \right\rangle + \frac{L\eta^2}{2(1 - \mu)^2}\CE_t\left[\left\|\frac{1}{M} \sum_{i=1}^Mg_{t,i} \right\|^2_2\right] \nonumber\\
& = & F(z_t) - \frac{\eta}{1 - \mu}\left\langle \nabla F(z_t), \nabla F(x_t) \right\rangle + \frac{L\eta^2}{2(1 - \mu)^2}\left[\|\nabla F(x_t)\|_2^2 + \CE_t\left[\left\|\frac{1}{M} \sum_{i=1}^Mg_{t,i} - \nabla F(x_t)\right\|^2_2\right]\right] \nonumber\\
& \leq & F(z_t) - \frac{\eta}{1 - \mu}\left\langle \nabla F(z_t), \nabla F(x_t) \right\rangle + \frac{L\eta^2}{2(1 - \mu)^2}\|\nabla F(x_t)\|_2^2 + \frac{L\eta^2\sigma^2}{2(1 - \mu)^2M}, \label{eq:mom_proof_eq_1}
\end{eqnarray}
where the second-to-last equality follows from $\CE[\|x - \CE[x]\|_2^2] = \CE[\|x\|_2^2] - \|\CE[x]\|_2^2$. Then, we bound the second term $-\left\langle \nabla F(z_t), \nabla F(x_t) \right\rangle$.
\begin{eqnarray}
-\left\langle \nabla F(z_t), \nabla F(x_t) \right\rangle & = & -\|\nabla F(x_t)\|_2^2 + \left\langle \nabla F(x_t) - \nabla F(z_t), \nabla F(x_t) \right\rangle  \nonumber\\
& \leq & -\left(1 - \frac{\rho}{2}\right)\|\nabla F(x_t)\|_2^2 + \frac{1}{2\rho}\left\| \nabla F(x_t) - \nabla F(z_t) \right\|_2^2 \label{eq:mom_proof_eq_2}
\end{eqnarray}
for any $0 < \rho < 2$. Then, we have
\begin{eqnarray}
\left\| \nabla F(x_t) - \nabla F(z_t) \right\|_2^2 & \leq & L^2\|x_t - z_t\|^2_2  \nonumber\\
& \leq & 2L^2\|x_t - \tx_t\|^2_2 + 2L^2\|\tx_t - z_t\|^2_2  \nonumber\\
& = & 2L^2\eta^2\left\|\te_t + \frac{1}{M} \sum_{i=1}^M e_{t,i}\right\|^2_2 + \frac{2L^2\eta^2\mu^4}{(1 - \mu)^2}\left\|\frac{1}{M} \sum_{i=1}^Mm_{t-1, i}\right\|^2_2  \nonumber\\
& \leq & \frac{16L^2\eta^2(1-\delta)G^2}{\delta^2(1-\mu)^2}\left[1 + \frac{16}{\delta^2}\right] + \frac{2L^2\eta^2\mu^4}{(1 - \mu)^2}\left\|\frac{1}{M} \sum_{i=1}^Mm_{t-1, i}\right\|^2_2, \label{eq:mom_proof_eq_3}
\end{eqnarray}
where in the last inequality we use Lemma~\ref{lemma:error_bound}. Let $A_{t-1} = \sum_{k=0}^{t-1}\mu^{t - 1 - k} = \frac{1 - \mu^{t}}{1 - \mu}$. Then, we bound the last term:
\begin{eqnarray}
\left\|\frac{1}{M} \sum_{i=1}^Mm_{t-1, i}\right\|^2_2 & = & A_{t-1}^2\left\|\sum_{k=0}^{t-1}\frac{\mu^{t - 1 - k}}{A_{t-1}}\frac{1}{M} \sum_{i=1}^Mg_{k, i}\right\|^2_2  \nonumber\\ 
& \leq & A_{t-1}^2\sum_{k=0}^{t-1}\frac{\mu^{t - 1 - k}}{A_{t-1}}\left\|\frac{1}{M} \sum_{i=1}^Mg_{k, i}\right\|^2_2  \nonumber\\ 
&= &A_{t-1}\sum_{k=0}^{t-1}\mu^{t - 1 - k}\left\|\frac{1}{M} \sum_{i=1}^Mg_{k, i}\right\|^2_2  \nonumber\\ 
& \leq & \frac{1}{1 - \mu}\sum_{k=0}^{t-1}\mu^{t - 1 - k}\left\|\frac{1}{M} \sum_{i=1}^Mg_{k, i}\right\|^2_2, \label{eq:mom_proof_eq_4}
\end{eqnarray}
where the first inequality follows from Jensen's inequality. 
Then, combining (\ref{eq:mom_proof_eq_1}), (\ref{eq:mom_proof_eq_2}), (\ref{eq:mom_proof_eq_3}), and (\ref{eq:mom_proof_eq_4}), we obtain
\begin{eqnarray*}
\lefteqn{\CE_t[F(z_{t+1})]} \nonumber\\
& \leq & F(z_t) - \left(\frac{\eta\left(2 - \rho\right)}{2(1 - \mu)} - \frac{L\eta^2}{2(1 - \mu)^2}\right)\|\nabla F(x_t)\|_2^2 
+  \frac{L^2\eta^3\mu^4}{\rho(1 - \mu)^4}\sum_{k=0}^{t-1}\mu^{t - 1 - k}\left\|\frac{1}{M} \sum_{i=1}^Mg_{k, i}\right\|^2_2 \\
&& + \frac{L\eta^2\sigma^2}{2(1 - \mu)^2M} 
+ \frac{8L^2\eta^3(1-\delta)G^2}{\rho\delta^2(1 - \mu)^3}\left[1 + \frac{16}{\delta^2}\right].
\end{eqnarray*}
Taking total expectation and telescoping this inequality from $0$ to $T - 1$, we obtain
\begin{eqnarray*}
\lefteqn{ \left(\frac{\eta\left(2 - \rho\right)}{2(1 - \mu)} - \frac{L\eta^2}{2(1 - \mu)^2}\right)\sum_{t=0}^{T-1}\CE[\|\nabla F(x_t)\|_2^2] } \nonumber\\
& \leq & \CE[F(z_0)] - \CE[F(z_{T})]
+  \frac{L^2\eta^3\mu^4}{\rho(1 - \mu)^4}\sum_{t=0}^{T-1}\sum_{k=0}^{t-1}\mu^{t - 1 - k}\CE\left[\left\|\frac{1}{M} \sum_{i=1}^Mg_{k, i}\right\|^2_2\right] \\
&& + \frac{L\eta^2\sigma^2T}{2(1 - \mu)^2M} 
+ \frac{8L^2\eta^3(1-\delta)G^2T}{\rho\delta^2(1 - \mu)^3}\left[1 + \frac{16}{\delta^2}\right] \\
& = & F(x_0) - \CE[F(z_{T})]
+  \frac{L^2\eta^3\mu^4}{\rho(1 - \mu)^4}\sum_{t=0}^{T-1}\sum_{k=0}^{t-1}\mu^{t - 1 - k}\CE\left[\left\|\nabla F(x_k)\right\|^2_2\right]  \\
&& +  \frac{L^2\eta^3\mu^4}{\rho(1 - \mu)^4}\sum_{t=0}^{T-1}\sum_{k=0}^{t-1}\mu^{t - 1 - k}\CE\left[\left\|\frac{1}{M} \sum_{i=1}^Mg_{k, i} - \nabla F(x_k)\right\|^2_2\right] \\
&& + \frac{L\eta^2\sigma^2T}{2(1 - \mu)^2M} 
+ \frac{8L^2\eta^3(1-\delta)G^2T}{\rho\delta^2(1 - \mu)^3}\left[1 + \frac{16}{\delta^2}\right] \\
& \leq & F(x_0) - F_*
+  \frac{L^2\eta^3\mu^4}{\rho(1 - \mu)^4}\sum_{t=0}^{T-1}\sum_{k=0}^{t-1}\mu^{t - 1 - k}\CE\left[\left\|\nabla F(x_k)\right\|^2_2\right]  \\
&& +  \frac{L^2\eta^3\mu^4\sigma^2T}{\rho(1 - \mu)^5M} + \frac{L\eta^2\sigma^2T}{2(1 - \mu)^2M} + \frac{8L^2\eta^3(1-\delta)G^2T}{\rho\delta^2(1 - \mu)^3}\left[1 + \frac{16}{\delta^2}\right] .
\end{eqnarray*}
Using double-sum trick, we get
\begin{eqnarray*}
\sum_{t=0}^{T-1}\sum_{k=0}^{t-1}\mu^{t - 1 - k}\CE\left[\left\|\nabla F(x_k)\right\|^2_2\right] & = & \sum_{k=0}^{T-2}\sum_{t=k+1}^{T-1}\mu^{t - 1 - k}\CE\left[\left\|\nabla F(x_k)\right\|^2_2\right] \\
& \leq & \frac{1}{1 - \mu}\sum_{k=0}^{T-2}\CE\left[\left\|\nabla F(x_k)\right\|^2_2\right] \\
& \leq & \frac{1}{1 - \mu}\sum_{k=0}^{T-1}\CE\left[\left\|\nabla F(x_k)\right\|^2_2\right].
\end{eqnarray*}
Rearranging the terms, we get
\begin{eqnarray}
\lefteqn{\sum_{t=0}^{T-1}\left(\frac{\eta\left(2 - \rho\right)}{2(1 - \mu)} - \frac{L\eta^2}{2(1 - \mu)^2} - \frac{L^2\eta^3\mu^4}{\rho(1 - \mu)^5}\right)\CE[\|\nabla F(x_t)\|_2^2] } \nonumber\\
& \leq & F(x_0) - F_*
 +  \frac{L^2\eta^3\mu^4\sigma^2T}{\rho(1 - \mu)^5M} + \frac{L\eta^2\sigma^2T}{2(1 - \mu)^2M} + \frac{8L^2\eta^3(1-\delta)G^2T}{\rho\delta^2(1 - \mu)^3}\left[1 + \frac{16}{\delta^2}\right]. \label{eq:mom_proof_eq_5}
\end{eqnarray}
Let $\eta \leq \frac{(2 - \rho)(1 - \mu)^2}{2L}$ and $\rho$ is selected such that $\rho \geq (2 - \rho)\mu^3$, we get
\begin{eqnarray}
\frac{\eta\left(2 - \rho\right)}{2(1 - \mu)} - \frac{L\eta^2}{2(1 - \mu)^2} - \frac{L^2\eta^3\mu^4}{\rho(1 - \mu)^5} \geq \frac{\eta\left(2 - \rho\right)}{4(1 - \mu)}. \label{eq:mom_proof_eq_6}
\end{eqnarray}
Hence, combining (\ref{eq:mom_proof_eq_5}) and (\ref{eq:mom_proof_eq_6}), and dividing by $T$, 
\begin{eqnarray*}
\frac{1}{T}\sum_{t=0}^{T-1}\CE[\|\nabla F(x_t)\|_2^2] 
& \leq & \frac{4(1 - \mu)}{\eta\left(2 - \rho\right)T}[F(x_0) - F_*] 
+ \frac{2L\eta\sigma^2}{(2-\rho)(1 - \mu)M}\left[1 + \frac{2L\eta\mu^4}{\rho(1 - \mu)^3}\right] \\
&& + \frac{32L^2\eta^2(1-\delta)G^2}{\rho(2 - \rho)\delta^2(1-\mu)^2}\left[1 + \frac{16}{\delta^2}\right].
\end{eqnarray*}
Let $\rho = 1$ and $\CE\left[\left\|\nabla F(x_o)\right\|^2_2\right] = \frac{1}{T}\sum_{t=0}^{T-1}\CE[\|\nabla F(x_t)\|_2^2]$, we obtain the result. 
\end{proof}



\section{Proof of Corollary~\ref{corollary:convergence-dist-ef-sgdm}}
 \begin{proof}
 Let $\eta = \frac{\gamma}{\frac{\sqrt{T}}{\sqrt{M}} + \frac{(1 - \delta)^{1/3}}{\delta^{2/3}}\left(1 + \frac{16}{\delta^2}\right)^{1/3}T^{1/3}}$ for some $\gamma > 0$. As $T \geq \frac{4\gamma^2L^2M}{(1 - \mu)^4}$, we have $\eta \leq \frac{(1 - \mu)^2}{2L}$ and 
\begin{eqnarray*}
\lefteqn{\CE\left[\left\|\nabla F(x_o)\right\|^2_2\right]}\\
& \leq & \left[\frac{1}{\gamma\sqrt{MT}} + \frac{(1 - \delta)^{1/3}}{\gamma\delta^{2/3}T^{2/3}}\left(1 + \frac{16}{\delta^2}\right)^{1/3}\right]4(1 - \mu)[F(x_0) - F_*] \\
&& + \frac{2L\gamma\sigma^2}{(1 - \mu)\sqrt{MT}}\left[1 + \frac{2L\gamma\mu^4\sqrt{M}}{(1 - \mu)^3\sqrt{T}}\right]  + \frac{32L^2\gamma^2(1-\delta)^{1/3}G^2}{\delta^{2/3}(1-\mu)^2T^{2/3}}\left[1 + \frac{16}{\delta^2}\right]^{1/3}\\
& = & \left[\frac{2(1 - \mu)}{\gamma}[F(x_0) - F_*]  + \frac{L\gamma\sigma^2}{1-\mu}\right]\frac{2}{\sqrt{MT}} 
+ \frac{4L^2\gamma^2\mu^4\sigma^2}{(1 - \mu)^4T} \\
&& + \frac{4(1 - \delta)^{1/3}\left[\frac{(1 - \mu)}{\gamma}[F(x_0) - F_*] + \frac{8L^2\gamma^2 G^2}{(1-\mu)^2}\right]}{\delta^{2/3}T^{2/3}}\left[1 + \frac{16}{\delta^2}\right]^{1/3}.
\end{eqnarray*} 
\end{proof}

\end{document}